\documentclass[twoside,11pt]{article}

\usepackage{jmlr2e}

\usepackage{amsbsy}
\usepackage{amsmath}
\usepackage{amssymb}
\usepackage{latexsym}
\usepackage{ifthen}
\usepackage{subfigure}
\usepackage{turnstile}
\usepackage{mathtools}
\usepackage{booktabs}
\usepackage{balance}
\usepackage{paralist}
\usepackage{url}
\usepackage{color}
\usepackage{graphicx}
\usepackage{appendix}
\usepackage{natbib}

\newtheorem{assumption}{Assumption}[section]

\newcommand{\vecc}{\boldsymbol{\operatorname{Vec}}}

\newcommand{\argmin}{\boldsymbol{\operatorname{argmin}}}

\newcommand{\cco}{\boldsymbol{\operatorname{\overline{co}}}}
\newcommand{\diag}{\textup{diag}}

\def\va{{\bf a}}

   \def\vx{{\bf x}}
\def\vy{{\bf y}} \def\vz{{\bf z}}

  \def\calC{\mathcal{C}}
\def\calD{\mathcal{D}}  \def\calF{\mathcal{F}}
 \def\calH{\mathcal{H}} \def\calI{\mathcal{I}}
\def\calJ{\mathcal{J}} \def\calK{\mathcal{K}} \def\calL{\mathcal{L}}
\def\calM{\mathcal{M}} \def\calN{\mathcal{N}} 
\def\calP{\mathcal{P}}  
  \def\calU{\mathcal{U}}
  
 \def\calZ{\mathcal{Z}}

\def \Vap{\varepsilon}

\def \C{\mathbb{C}}

\def \wC{\widehat{\C}}

\def \bmu{\boldsymbol{\mu}}

\def \bz{\breve{\mathbf{z}}}

\def \bx{\breve{\mathbf{x}}}

\def \oM{\overline{\mathcal{M}}}
\def \wy{\widehat{\mathbf{y}}}
\def \wx{\widehat{\mathbf{x}}}

\def \R{\mathbb{R}}

\def \N{\mathbb{N}}
\def \bcalC{\breve \calC}
\def \OZ{\overline{\mathcal{Z}}}
\def \y{\vy}
\def \myQ{Q} % Handle for all instances of the multi-agent objective. Change back to \calQ here if desired.

\ShortHeadings{}{Kar and Swenson}
\firstpageno{1}

\begin{document}

\title{Clustering with Distributed Data}

\author{\name Soummya Kar \email soummyak@andrew.cmu.edu\\
\addr Department of Electrical and Computer Engineering,\\
Carnegie Mellon University, Pittsburgh, PA 15213
\AND
\name Brian Swenson \email bswenson@princeton.edu\\
\addr Department of Electrical Engineering,\\
Princeton University, Princeton, NJ 08540}

\editor{}

\maketitle

\begin{abstract}%   <- trailing '%' for backward compatibility of .sty file
We consider $K$-means clustering in networked environments (e.g., internet of things (IoT) and sensor networks) where data is inherently distributed across nodes and processing power at each node may be limited. We consider a clustering algorithm referred to as \emph{networked $K$-means}, or \emph{$NK$-means}, which relies only on local neighborhood information exchange. Information exchange is limited to low-dimensional statistics and not raw data at the agents.
The proposed approach develops a parametric family of multi-agent clustering objectives (parameterized by $\rho$) and associated distributed $NK$-means algorithms (also parameterized by $\rho$). The $NK$-means algorithm with parameter $\rho$ converges to a set of fixed points relative to the associated multi-agent objective (designated as \emph{generalized minima}). By appropriate choice of $\rho$, the set of generalized minima may be brought arbitrarily close to the set of Lloyd's minima. Thus, the $NK$-means algorithm may be used to compute Lloyd's minima of the collective dataset up to arbitrary accuracy.
\end{abstract}


\section{Introduction}
\label{introduction}

$K$-means clustering is a tool of fundamental importance in computer science and engineering with a wide range of applications \citep{jain2010data,wu2008top}.
In this paper we are interested in studying algorithms for $K$-means clustering in modern network-based computing environments where data is naturally distributed across nodes and computational power at each node may be limited. Settings of interest include the internet of things (IoT) \citep{xia2012internet}, vehicular networks \citep{hartenstein2008tutorial}, sensor networks \citep{yick2008wireless}, and device-to-device 5G cellular networks \citep{tehrani2014device}.\footnote{We note that there is some disparity between terminology across fields and the meaning of these terms is sometimes conflated. }
A bevy of recent research has focused on developing decentralized algorithms for machine learning in such network-based settings \citep{jakovetic2018convergence,jiang2017collaborative,lian2017can,sahu2018distributed,tang2018d}.

The past decade has seen tremendous growth in research and infrastructure development for the (inherently centralized) cloud computing framework. While beneficial in many applications, cloud computing has limitations and there is a strong trend towards handling more computation and data storage at the periphery of the network on user devices or small data centers close to user devices \citep{shi2016edge,hu2015mobile,satyanarayanan2017emergence}.

The motivation for this trend is driven by several factors. First, given the proliferation in data generated by user devices, it can be impractical to communicate all data to centralized locations for evaluation. For example, current self-driving automobiles generate as much as 1GB of data per second \citep{shi2016edge}. Beyond self-driving vehicles, it is estimated that there will be 50 billion connected IoT devices by the year 2020 \citep{evans2011internet};
handling all data and computations generated by such devices in a centralized fashion results in high latency and an impractical burden on the network infrastructure.
A second motivation comes from the perspective of user privacy and data security: many users are opposed to sharing (possibly sensitive) data with companies or storing such data in centralized locations \citep{house2012consumer}; moreover,
 storing and processing data in a distributed fashion can mitigate security risks \citep{mcmahan2016communication}.

In settings involving wireless device-to-device networks (e.g., sensor networks, vehicular networks, or device-to-device 5G cellular networks), data is often naturally distributed between nodes of the network and distributed computation schemes are critical in ensuring robustness and extending network lifetime \citep{shnayder2004simulating,yu2004energy}.

Formally, in this paper we focus on developing efficient $K$-means clustering algorithms for multi-agent systems in which (i) each agent possesses a set of data points $\calD_i\subset\R^p$ and (ii) agents may exchange information over some preassigned communication graph $G=(V,E)$. At a high level, the objective is for agents to cooperatively cluster the joint dataset $\calD = \calD_1\cup\cdots\cup \calD_N$. (The details of this setting will be made precise in Section \ref{sec:dist-framework}.)

Before proceeding further, we comment on the distributed computation architecture used in our multi-agent setup. Our distributed setup consists of a collection of $M$ agents with (limited) storage, computation, and communication capabilities. Agents may have access to local data (generated by a local data source or acquired through sensing) and agents interact with other agents, by means of message exchanges, over a preassigned connected communication network (possibly sparse) to achieve a common computation or inference goal. (For instance, in a distributed function computation problem, agents may be interested in computing a function of their collective data.) We direct the readers to~\citep{tsitsiklisphd84,tsitsiklisbertsekasathans86,Bertsekas-survey,Kushner-dist,cybenko89,degroot1974reaching} for early work that focuses on a range of collaborative distributed computing and decision-making in such setups. Over the last decade and more recently, there has been renewed interest in such setups and variants, motivated by applications in computing, learning, and optimization in IoT-type setups, ad-hoc networks, peer-to-peer processing, to name a few. Often a prominent feature in these setups is that, to achieve a common decision-making objective on the collective data (could be static or streaming), agents do not directly exchange raw data (possibly very high dimensional) with each other but iterate (compute and communicate) over appropriate \emph{local statistics}, preferably of lower dimension and complexity than their raw data, with the aim of converging to the common quantity of interest. For instance, in a distributed learning or optimization setup, agents may maintain a local copy of the parameter or optimizer that minimizes a global risk function based on their collective data, and these local copies are iteratively updated by means of local computation and inter-agent message exchanges with the goal of converging to the desired optimizer or a reasonable approximation. For a sample of recent advances in the field, we direct the readers to~\citep{BoydGossip,dimakis2010gossip,KarMouraRamanan-Est-2008,Kar-QD-learning-TSP-2012,Giannakis-LMS,Sayed-LMS,
RabbatDistributedStronglyCVX,SayedStochasticOpt,DistributedMirrorDescent,Kozat,NedicStochasticPush,Ram-Nedich-Siam,sahu2015distributed,arxivVersion,ma2015adding,ma2015partitioning,heinze2016dual,zhang2013divide} that study a broad range of distributed decision-making problems in multi-agent setups ranging from distributed parameter estimation to online stochastic optimization.

In general, the problem of computing a globally optimal $K$-means clustering is NP-hard. A popular relaxed solution concept is that of a Lloyd's minimum \citep{lloyd1982least}; in this paper, we will focus on developing distributed algorithms for computing Lloyd's minima. More specifically, in the current distributed $K$-means context, we adopt a multi-agent viewpoint as described above and present a distributed algorithm, referred to as the \emph{networked $K$-means} (or $NK$-means) of the consensus+innovations type~\citep{KarMouraRamanan-Est-2008,kar2013consensus+}. The consensus+innovations type approach is well suited to such distributed setups (see also the relevant family of diffusion algorithms~\citep{Sayed-LMS}) in which each agent maintains a local copy (estimate) of the desired $K$ cluster heads for the collective network data and iteratively updates their local copies by simultaneously assimilating the estimates of the neighboring agents\footnote{The neighborhood of an agent here refers to its communication neighborhood and consists of those agents that can directly communicate with the agent.} (the consensus or agreement potential) and taking a \emph{refinement} step using their local data (the innovation potential).  \\

\noindent \textbf{Main Contributions}. The main contributions of the paper are the following:

1. We propose an algorithm (or, more precisely, a parametric class of algorithms) for $K$-means clustering in networked multi-agent settings with distributed data.
We refer to this algorithm as \emph{networked $K$-means}, or $NK$-means in short.

The proposed class of algorithms is parameterized by $\rho\in \N_+$. Solutions obtained by the algorithm may be brought arbitrarily close to the set of Lloyd's minima by appropriate choice of $\rho$. Our next contribution makes this relationship precise.
%\footnote{\textcolor{red}{Note, BS: This issue may be confusing. When people hear Lloyd's minimum, do they think of the cluster centers or the partition of the dataset? If the partition, then saying we  approximate it arbitrarily well doesn't really make sense. Clarify this.}}

2. We introduce the notion of a generalized Lloyd's minimum---a generalization of the classical Lloyd's minimum adapted to the multi-agent setting. As with the $NK$-means algorithm, the set of generalized Lloyd's minima are parameterized by a parameter $\rho\in\N_{+}$. We show that an instantiation of the $NK$-means algorithm with parameter $\rho$ converges to the set of generalized Lloyd's minima with parameter $\rho$ (Theorem \ref{th:conv}). Moreover, we show that as $\rho\to\infty$, the set of generalized Lloyd's minima approaches the set of classical Lloyd's minima (Theorem \ref{lm:conv_to_Z}).
%Thus, the limit points of the $NK$-means algorithm (with parameter $\rho$) may be made arbitrarily close to the set of Lloyd's minima by appropriate choice of $\rho$. The tradeoff is that the rate of convergence of the algorithm is slower for larger values of $\rho$.

Theorems \ref{lm:conv_to_Z} and \ref{th:conv} together show that the $NK$-means algorithm can be used to compute Lloyd's minima up to arbitrary accuracy.

3. Generalized Lloyd's minima are obtained as (generalized) minima of a $\rho$-relaxed multi-agent $K$-means objective denoted by $\myQ^\rho$. In addition to characterizing the behavior of the set of generalized Lloyd's minima as $\rho\to\infty$, it is shown that the set of global minima of $\myQ^\rho$ converges to the set of global minima of the classical $K$-means objective as $\rho\to\infty$, and a characterization of the rate of convergence is given (Theorem \ref{lm:gloptconv1}).

A formal presentation of these results and additional discussion, including the tradeoffs inherent in the choice of $\rho$, will be given in Section \ref{sec:main_res}.\\

\noindent \textbf{Related Work}.
The problem of clustering in a network-based setting with distributed data was considered in \citep{bandyopadhyay2006clustering} using an approach in which data is replicated at all nodes. Similar approaches were taken in \citep{datta2009approximate,datta2006distributed,datta2006k}. In contrast to these works, the present paper does not rely on replicating data across nodes, which can be impractical in large-scale settings and jeopardize or violate user privacy. The work \citep{di2013fault} considers algorithms for $K$-means clustering in this setting with promising experimental results but does not provide any theoretical analysis. The work  \citep{oliva2013distributed} considers $K$-means clustering in a sensor network setting in which each node holds a single data point (but not a data set). The algorithm relies on finite-time consensus techniques to mimic the centralized $K$-means algorithm. This technique is extended in \citep{qin2017distributed} to consider improved initialization schemes per \citep{arthur2007k}. The work \citep{jagannathan2005privacy} considers a privacy-preserving protocol for $K$-means clustering when data is distributed between two parties.

When the underlying dataset is large, the problem of finding even an approximately optimal solution to the $K$-means clustering problem can be computationally demanding. From this perspective, several papers have considered methods for parallelizing the computation by distributing data among several machines.
The work \citep{balcan2013distributed} considers an approach in which each node computes a coreset (i.e., a subset of the data that serves a good approximation for the purpose of clustering; see \citep{har2004coresets}) of its local data. The individual coresets are then transmitted via an overlaid communication graph to a central node, which determines an approximately optimal solution for the full $K$-means problem. (Alternatively, the network may be flooded with the individual coresets and the solution computed at each node.) The suboptimality of the solution is bounded, and an estimate on the required number of communications is established.
The work \citep{bateni2014distributed} follows a similar approach to \citep{balcan2013distributed} but incorporates balancing constraints.
The works \citep{malkomes2015fast} and \citep{awasthi2017general} propose distributed clustering algorithms that are robust to outliers.
Our work differs from these in several aspects. First, our approach does not require any centralized node or flooding of the network. Moreover, it is based on consensus+innovations techniques,
which have been shown to be robust to errors and disturbances common in network-based settings (e.g., link failures, changes in communication topology, and agents entering or exiting the network) \citep{KarMouraRamanan-Est-2008}. Such approaches are also robust in that there is no central point of failure. Furthermore, our approach does not require explicit sharing of any data points, which can compromise privacy. From a broader perspective, the fundamental motivation for our work differs from these in that we are motivated by applications in IoT-type networks where data is naturally distributed and computational power at each node may be limited.

%A vast body of recent research has focused on algorithms for distributed optimization.
A closely related line of research considers algorithms for distributed optimization. The majority of this research has focused on convex problems \citep{nedic2009distributed,boyd2011distributed,rabbat2004distributed,mota2013d,sahu2018distributed}, though recent research has begun to investigate non-convex problems \citep{scutari2017parallel,sun2016distributed,tatarenko2017non}. The present work may be seen as a contribution in this area in that it develops an algorithm for distributed optimization with a  non-convex non-smooth objective.

%The majority of research in this direction considers objective functions which are separable (across agents) and convex, e.g., \citep{nedic2009distributed,boyd2011distributed,rabbat2004distributed,mota2013d,sahu2018distributed-1}. A more recent research thrust has begun to investigate distributed optimization with non-convex objective functions \citep{scutari2017parallel,sun2016distributed,tatarenko2017non}. The present work contributes to this literature by  developing an algorithm for distributed optimization with a  non-convex non-smooth objective.

We remark that, from a technical perspective, our work differs from many works on $K$-means clustering (centralized or otherwise) in that we make no assumptions on the data beyond the minimum assumption that the dataset has at least $K$-distinct datapoints. In particular, we do not assume that datapoints are distinct. This is necessary to handle degenerate cases that may arise in the multi-agent distributed-data framework. For example, if two agents have access to the same data source there can be redundancy in the collective dataset. This introduces technical challenges into the analysis of algorithms for $K$-means clustering, which are non-trivial to address.

%We remark that the solution obtained by the $K$-means algorithm can be sensitive to initialization. In order to mitigate this issue, \citep{arthur2007k} and \citep{ostrovsky2006effectiveness} proposed initialization schemes which have been shown to perform well in practice. In \citep{bahmani2012scalable} this was extended to large-scale parallel computing frameworks. In this paper we will not consider initialization schemes for the network-based distributed data framework, but we note that adapting similar techniques to the network-based framework may be a productive direction for future research.

%The approach for distributed clustering developed in this paper is based on \emph{consensus + innovations} techniques \citep{kar2013consensus+,hug2015consensus+,kar2012distributed}. In the present work our focus is on studying convergence of the $NK$-means algorithm to the set of Lloyd's minima. While we do not formally address robustness of the algorithm in the present work, we note that consensus + innovations techniques tend to be inherently robustness to errors and disturbances common in a network-based setting (e.g., link failures, changes in communication topology, and agents entering or exiting the network) \citep{kar2012distributed}.

\vspace{1em}
\noindent\textbf{Organization.}
The remainder of the paper is organized as follows. Section \ref{sec:notation} sets up notation. Section \ref{sec:prob-form} reviews the classical $K$-means problem and the notion of a Lloyd's minimum. Section \ref{sec:dist_K_means} introduces our generalized multi-agent $K$-means clustering objective and introduces the notion of a generalized Lloyd's minimum. Section \ref{sec:dist-algorithm} presents the $NK$-means algorithm.  Section \ref{sec:main_res} summarizes the main results of the paper. Section \ref{sec:sims} presents a simple illustrative example. The remaining sections are devoted to the analysis of the $NK$-means algorithm and properties of generalized Lloyd's minima. Section \ref{sec:NKmeansconv} proves that the $NK$-means algorithm with parameter $\rho$ converges to the set of generalized Lloyd's minima with parameter $\rho$. Section \ref{sec:genminimaconv}
%considers properties of generalized Lloyd's minima, and, in particular,
shows that the set of generalized Lloyd's minima with parameter $\rho$ converges to the set classical Lloyd's minima as $\rho\to\infty$. Section \ref{sec:global-min-convergence} shows that global minima of the generalized multi-agent $K$-means objective  converge to the set of global minima of the classical $K$-means objective as $\rho\to\infty$.

\section{Notation} \label{sec:notation}
%The following notation will be used throughout the paper.
We denote by $|\mathcal{X}|$ the cardinality of a finite set $\mathcal{X}$. Let $\R_{+} = [0,\infty)$ and $\N_{+} = \{1,2,\ldots,\}$. Let $\|\cdot\|$ denote the $\ell_{2}$ norm on $\mathbb{R}^{p}$. Given a set $S\subset \R^p$ and a point $x\in \R^p$, let $d(x,S) = \inf_{y\in S} \|x-y\|$.
For $n\in \N_{+}$, let $I_n$ denote the identity matrix of size $n$.
Given a set $S\subset\mathbb{R}^{p}$, let $\cco(S)$ denote the closed convex hull of $S$.
Given a finite set of vectors $\vx_k\in \R^p$, let $\vecc_k(\vx_k)$ denote the vector stacking all vectors $\vx_k$ across the index set.

Given a set $S\subset \R^p$, we say that a collection of sets $\calP=\{\calP^1,\ldots,\calP^K\}$, $\calP^k\subset \R^p$ is a \emph{partition of size $K$} of $S$ if $\bigcup_{k=1}^K \calP^k = S$ and $\calP^k\cap \calP^{k'} = \emptyset$ for all $k,k'\in \{1,\ldots,K\}$, $k\not=k'$. Note that some sets in a partition may be empty.

Given a (undirected) graph $G=(V,E)$ on $M$ agents or nodes, with node set $V$ indexed as $V=\{1,\cdots,M\}$, the associated graph Laplacian is given by $L = D - A$, where $D$ is the degree matrix of the graph and $A$ is the adjacency matrix. The set of neighbors of agent $m$ is given by $\Omega_m = \{i\in V:(i,m)\in E\}$. The Laplacian $L$ is a positive semidefinite matrix. Denoting by $0=\lambda_{1}(L)\leq\lambda_{2}(L)\leq\cdots\leq\lambda_{M}(L)$ the eigenvalues of $L$, we note that $\lambda_{2}(L)>0$ if and only $G$ is connected. A review of spectral graph theory can be found in \citep{chung1997spectral}.

Several symbols will be introduced through the course of presenting and proving the results in the paper. For convenience, a reference list of frequently used symbols is included in Appendix A.

%\vspace{2em}
%Other sets:
%\begin{itemize}
%\item $\calU_m(t)$ is the set of ``viable'' partitions of $\calD_m$ at stage $t$ under the distributed algorithm.\footnote{I think this notation should be changed so it depends explicitly on $\vx(t)$. I.e., $\calU_{\vx(t)}$ as defined later.}
%\item $\C_m$ is the set of all partitions of size $K$ of data $\calD_m$
%\item $\K$ is set of all feasible pairs $(\vx,\calC) = \{(\vx_1,\ldots,\calC_1),\ldots,(\vx_M,\calC_M)\}$, where each $\vx_m$ gives set of cluster centers at agent $m$ and $\calC_m \in \calC_m$.
%\item Given tuple of cluster centers at each agent $\vx \in \R^{MKp}$, $\calU_{\vx}$ is set of ``viable'' partitions of each agents data $\calD_m$. An element of $\calU_{\vx}$ lives in $\C_1\times\cdots\times C_M$
%\item $\overline{\calU}_{\vx}$ is subset of $\calU_{\vx}$ which also satisfies generalized minima condition at $\vx$ (could be empty)
%    \begin{itemize}
%    \item Note: $\vx\in \R^{MKp}$ is a generalized minimum iff $\overline{\calU}_{\vx}$ is nonempty
%    \end{itemize}
%\end{itemize}

\label{notgraph}

\section{Problem Formulation and Preliminaries}
\label{sec:prob-form}
%In this section we introduce the multi-agent setup and review some preliminaries in (centralized) $K$-means clustering.\\

We will now review the classical $K$-means clustering problem, Lloyd's algorithm, and the notion of a Lloyd's minimum. After reviewing these concepts in Sections \ref{sec:k-means-classical}--\ref{sec:lloyds-review}, Section \ref{sec:dist-framework} formally states the $K$-means problem in the networked multi-agent distributed-data framework.

%Section \ref{sec:k-means-classical} reviews the classical $K$-means problem. Section \ref{sec:lloyds-review} reviews Lloyd's algorithm and the notion of Lloyd's minima. Section \ref{sec:dist-framework} formally states the $K$-means problem in the multi-agent distributed-data framework.

\subsection{The Classical $K$-Means Clustering Problem} \label{sec:k-means-classical}
Let $\calD$ denote a finite collection of data points taking values in an Euclidean space $\R^p$, $p\geq 1$ and let $N= |\calD|$.
%Let there be $M$ agents (possibly data centers) each having private access to a of a set of data points taking values in an Euclidean space $\mathbb{R}^{p}$, $p\geq 1$. Denote by $\mathcal{D}_{m}$ the collection of data points at agent $m$, where $m=1,\cdots,M$, and by $\mathcal{D}$ the totality of data points, i.e., $\mathcal{D}=\mathcal{D}_{1}\cup\mathcal{D}_{2}\cup\cdots\cup\mathcal{D}_{M}$. Let $N_m = |\calD_m|$ denote the number of data points at agent $m$.
Given some $K\in \{2,\ldots,N\}$, the $K$-means clustering cost for a tuple of so-called cluster heads\footnote{Unless stated otherwise, we use the terms cluster head and cluster center interchangeably.} $\vx = \{\vx^1,\ldots,\vx^K\}$, $\vx^k \in \R^p$, $k=1,\ldots, K$ is given by
\begin{equation}
\label{Kmeanscost}
\mathcal{F}(\vx)=\sum_{\mathbf{y}\in\mathcal{D}}\min_{k=1,\cdots,K}\|\y-\vx^{k}\|^{2}.
\end{equation}
%where $\mathbf{y}$ denotes the $n$-th data point, $n=1,\cdots,N_{m}$, at the $m$-th agent and $\vx=\vecc(\{\vx^{k}\}_{k=1}^K)$, i.e., $\vx$ stacks the quantities $\{\vx^{k}\}_{k=1}^{K}$ into a vector.
The $K$-means clustering problem consists of solving the non-convex optimization problem (see~\citep{selim1984k})
\begin{equation}
\label{Kmeans}
\inf_{\{\vx^1,\ldots,\vx^K\}}\mathcal{F}(\vx),
\end{equation}
where the optimization is taken over the set of feasible cluster heads $\vx^k\in \R^p$, $k=1,\ldots,K$.
We remark that a global minimizer exists (i.e., the infimum in \eqref{Kmeans} is attainable) but need not be unique.

A tuple of cluster heads $\vx = \{\vx^1,\ldots,\vx^K\}$ induces a natural partitioning (not necessarily unique) of the dataset under the cost function \eqref{Kmeans}. In particular, for any $K$-tuple of cluster heads $\vx$ we may associate with $\vx$ any partition $\calP = \{\calP^1,\ldots,\calP^K\}$ of $\calD$ such that for each $y\in\calD$ there holds
\begin{equation} \label{eq:partition-condition}
y\in \calP^k \implies \|\vy - \vx^k\| \leq \|\vy - \vx^{k'}\| \mbox{ for all } k'=1,\ldots,K.
\end{equation}
Note that for any such partition we have $\calF(\vx) = \sum_{k=1}^K \sum_{\vy\in\calP^k} \|\vy - \vx^k\|^2$.

%A (global) minimizer of~\eqref{Kmeans} constitutes a set of so-called $K$-means of the population $\mathcal{D}$ and naturally induces a partitioning of the data set $\mathcal{D}$ into $K$ clusters such that the sum of squared $\ell_{2}$ distances of the data points from their respective cluster heads is minimized.
%Formally, denoting by $\mathcal{P}=\{\mathcal{P}^{1},\mathcal{P}^{2},\cdots,\mathcal{P}^{K}\}$ a partition of $\mathcal{D}$ into $K$ subsets, i.e., $\mathcal{P}^{k}\subset\mathcal{D}$ for all $k$, $\mathcal{P}^{k}\cap\mathcal{P}^{\acute{k}}=\emptyset$ for all $k\neq \acute{k}$ and $\mathcal{P}^{1}\cup\mathcal{P}^{2}\cup\cdots\cup\mathcal{P}^{K}=\mathcal{D}$, and a set $\mathbf{z}=\{\mathbf{z}^{k}\}_{k=1}^{K}$ of $K$ points in $\mathbb{R}^{p}$ (the potential cluster heads),

%In order to state the optimization problem \eqref{Kmeans} in a more tractable form it is advantageous to define an alternative cost function.
Given an arbitrary partition $\calP=\{\calP^1,\ldots,\calP^K\}$ of the dataset $\calD$ and a set of cluster heads $\vx = \{\vx^1,\ldots,\vx^K\}$,
let $\mathcal{H}:(\vx,\mathcal{P})\mapsto\mathbb{R}_{+}$ be the cost function
\begin{equation}
\label{costH}
\mathcal{H}(\vx,\mathcal{P})=\sum_{k=1}^{K}\sum_{\y\in\mathcal{P}^{k}}\|\y-\vx^{k}\|^{2}.
\end{equation}
Intuitively, for a given tuple $(\vx,\mathcal{P})$,
%of $K$ points $\mathbf{z}=\{\mathbf{z}^{1},\cdots,\mathbf{z}^{K}\}$ in $\mathbb{R}^{p}$ and a partition $\mathcal{P}=\{\mathcal{P}^{1},\mathcal{P}^{2},\cdots,\mathcal{P}^{K}\}$ of $\mathcal{P}$,
the quantity $\mathcal{H}(\vx,\mathcal{P})$ reflects the $\ell_{2}$ distortion cost of partitioning the data $\mathcal{D}$ into clusters $\calP^1,\ldots,\calP^K$ and, for each $k=1,\cdots,K$, taking $\vx^{k}$ to be the representative of all the data points in the $k$-th cluster $\mathcal{P}^{k}$. It is readily shown that
the $K$-means clustering problem as posed in~\eqref{Kmeans} is equivalent to the problem
\begin{equation}
\label{Hmin}
\inf_{\{\vx^{1},\cdots,\vx^{K}\},\{\mathcal{P}^{1},\cdots,\mathcal{P}^{K}\}}\mathcal{H}(\vx,\mathcal{P})
\end{equation}
of jointly minimizing the cost~\eqref{costH} over all tuples $(\vx,\mathcal{P})$ of cluster heads and partitions, in that, if $(\vx^{\ast},\mathcal{P}^{\ast})$ is a global minimizer of~\eqref{Hmin}, then $\vx^{\ast}$ is a global minimizer of~\eqref{Kmeans}; and, conversely, if $\vx^{\ast}$ is a global minimizer of~\eqref{Kmeans}, then the tuple $(\vx^{\ast},\mathcal{P}^{\ast})$ is a global minimizer of~\eqref{Hmin}, where $\mathcal{P}^{\ast}$ may be any partition satisfying \eqref{eq:partition-condition} with $\vx = \vx^\ast$.
Given the equivalence between the optimization problems~\eqref{Kmeans} and~\eqref{Hmin}, the latter will also be referred to as the $K$-means clustering problem.

The following mild assumption will be enforced throughout.
\begin{assumption}
\label{ass:data_size} The collective data set $\mathcal{D}$ consists of at least $K$ distinct data points.
\end{assumption}

Under this assumption the following properties hold for any global minimizer of \eqref{Kmeans} or \eqref{Hmin}.
\begin{proposition}
\label{prop:globminprop} Let Assumption~\ref{ass:data_size} hold. Then any global minimizer $\vx^{\ast}=\{\vx^{\ast 1},\cdots,\vx^{\ast K}\}$ of~\eqref{Kmeans} satisfies $\vx^{\ast k}\in\cco(\mathcal{D})$ for all $k$.
%and $\vx^{\ast k}\neq \vx^{\ast \acute{k}}$ for all $k,\acute{k}\in\{1,\cdots,K\}$ and $k\neq \acute{k}$.
Furthermore, a global minimizer consists of $K$ distinct cluster centers, i.e., $x^{\ast,k} \not= x^{\ast,k'}$ for all $k,k'\in\{1,\ldots,K\}$, $k\not=k'$.

Similarly, if a tuple $\{\vx^{\ast},\mathcal{P}^{\ast}\}$ is a global minimizer of~\eqref{Hmin}, we have that the $K$ cluster centers $\{\vx^{\ast 1},\cdots,\vx^{\ast K}\}$ are distinct, $\vx^{\ast k}\in\cco(\mathcal{D})$ for all $k$, and additionally $\mathcal{P}^{\ast k}\neq\emptyset$ for all $k\in\{1,\cdots,K\}$.
\end{proposition}

The proof of this proposition is straightforward and omitted for brevity.

\subsection{ Lloyd's algorithm and local minima of $K$-means} \label{sec:lloyds-review}

As noted earlier, the $K$-means clustering objective~\eqref{Kmeans} is non-convex. Indeed the problem of finding a global optimal $K$-means clustering with respect to \eqref{Hmin} is NP hard \citep{aloise2009np}.
Hence, in practice, it may only be possible to obtain local minima or approximate solutions of~\eqref{Kmeans}.

A commonly adopted notion of \emph{approximate solutions} of the $K$-means clustering problem is that of Lloyd's minima \citep{friedman2001elements}. The set of Lloyd's minima consist of a set of approximate solutions to~\eqref{Kmeans} which are obtained as limit points of an iterative procedure, referred to as Lloyd's algorithm \citep{lloyd1982least} (or sometimes simply as the \emph{$k$-means} algorithm).
Specifically, consider the optimization formulation~\eqref{Hmin} and let $(\vx_{0},\mathcal{P}_{0})$ be an initial solution estimate. Lloyd's algorithm is an iterative procedure that generates a sequence $\{(\vx_{t},\mathcal{P}_{t})\}$ of tuples of cluster heads and clusters starting with the initial estimate $(\vx_{0},\mathcal{P}_{0})$, where, at each iteration $t\in \N_+$, the tuple $(\vx_{t},\mathcal{P}_{t})$ is updated according to the following two-step procedure:
\begin{enumerate}[(i)]
\item \textbf{Reassignment of clusters}. For each data point $\y\in\mathcal{D}$, let $\omega_{\vy}(t)\in\{1,\cdots,K\}$ be the index of the cluster it belongs to at iteration $t$. If there exists $\acute{k}\in\{1,\cdots,K\}$ such that $\|\y-\vx_{t}^{\acute{k}}\|< \|\y-\vx_{t}^{\omega_{\vy}(t)}\|$, then $\y$ is moved (reassigned) to cluster $\mathcal{P}^{\acute{k}}$. (Note, if there exist multiple such $\acute{k}$'s, the point $\y$ may be arbitrarily assigned to any one of them.) The above reassignment is performed for each point $\y\in\mathcal{D}$ which generates the new set $\mathcal{P}_{t+1}$ of clusters.\\

\item \textbf{Center update}. Subsequently, the new cluster centers $\vx_{t+1}=\{\vx_{t+1}^{1},\cdots,\vx_{t+1}^{K}\}$ are obtained as the centroids of the respective clusters, i.e.,
\begin{equation}
\label{lloyd_center_up}
\vx_{t+1}^{k}=\frac{1}{|\mathcal{P}_{t+1}^{k}|}\sum_{\y\in\mathcal{D}}\y~~~\forall k\in\{1,\cdots,K\}.
\end{equation}
More precisely, the above center update~\eqref{lloyd_center_up} is performed only if $|\mathcal{P}_{t+1}^{k}|\neq 0$; otherwise, for definiteness, we set $\vx_{t+1}^{k}=\vx_{t}^{k}$.
\end{enumerate}

It is important to note that the reassignment and center update steps are locally optimal, in the sense that, for all $t$,
\begin{equation}
\label{lloyd_up_1}
\mathcal{P}_{t+1}\in\argmin_{\mathcal{P}}\mathcal{H}(\vx_{t},\mathcal{P})~~\mbox{and}~~\vx_{t+1}\in\argmin_{\vx}\mathcal{H}(\vx,\mathcal{P}_{t+1}).
\end{equation}
%and
%\begin{equation}
%\label{lloyd_up_2}
%\vx_{t+1}\in\argmin_{\vx}\mathcal{H}(\vx,\mathcal{P}_{t+1}),
%\end{equation}
%i.e., for a fixed set $\vx_{t}$ of cluster centers, $\mathcal{P}_{t+1}$ defines the optimal clustering minimizing the cost~\eqref{costH}, whereas, given a set of clusters $\mathcal{P}_{t+1}$, the corresponding set of centroids $\vx_{t+1}$ minimize the distortion~\eqref{costH}.
As a consequence, the clustering cost improves at every step, i.e., for all $t$,
\begin{equation}
\label{lloyd_up3}
\mathcal{H}(\vx_{t+1},\mathcal{P}_{t+1})\leq \mathcal{H}(\vx_{t},\mathcal{P}_{t+1})\leq\mathcal{H}(\vx_{t},\mathcal{P}_{t}).
\end{equation}
Let $\mathcal{L}$ denote the set of fixed points of Lloyd's algorithm, i.e., a pair $(\breve{\vx},\breve{\mathcal{P}})\in\mathcal{L}$ if and only if
\begin{equation}
\label{lloyd_up4}
\breve{\mathcal{P}}\in\argmin_{\mathcal{P}}\mathcal{H}(\breve{\vx},\breve{\mathcal{P}})~~\mbox{and}~~\breve{\vx}\in\argmin_{\vx}\mathcal{H}(\vx,\breve{\mathcal{P}}).
\end{equation}
It follows, by the step-wise cost-improvement property~\eqref{lloyd_up3} and the fact that the number of possible partitions of the data set $\mathcal{D}$ is finite, that the sequence $\{\vx_{t},\mathcal{P}_{t}\}$ generated by Lloyd's algorithm converges to a fixed point in $\mathcal{L}$ in finite time. The particular fixed point to which the algorithm converges is heavily dependent on the initial choice $(\vx_{0},\mathcal{P}_{0})$ of cluster center and cluster \citep{milligan1980examination}.

The set $\mathcal{L}$ is also referred to as the set of Lloyd's minima for the optimization formulation~\eqref{Hmin}. (For a detailed discussion on aspects of local optimality and stability of Lloyd's minima, we refer the reader to~\citep{selim1984k}.) Further, denote by $\mathcal{Z}$ the subset of cluster centers given by
\begin{equation}
\label{lloyd_up5}
\mathcal{Z}=\left\{\breve{\vx}~|~\exists~\breve{\mathcal{P}}~\mbox{such that}~(\breve{\vx},\breve{\mathcal{P}})\in\mathcal{L}\right\}.
\end{equation}
By convention, we will also refer to the set $\mathcal{Z}$ as the set of Lloyd's minima for the $K$-means clustering formulation~\eqref{Kmeans}.
\begin{remark}
\label{rem:Lloyd_min} It may be noted that, if $\breve{\vx}\in\mathcal{Z}$, then $\mathcal{F}(\breve{\vx})=\mathcal{H}(\breve{\vx},\breve{\mathcal{P}})$ for any $\breve{\mathcal{P}}$ such that $(\breve{\vx},\breve{\mathcal{P}})\in\mathcal{L}$; conversely, it follows that, if a pair $(\breve{\vx},\breve{\mathcal{P}})\in\mathcal{L}$, then $\mathcal{F}(\breve{\vx})=\mathcal{H}(\breve{\vx},\breve{\mathcal{P}})$. We also note that the set $\mathcal{L}$ may consist of tuples $(\breve{\vx},\breve{\mathcal{P}})$ for which one or more of the sets $\{\breve{\mathcal{P}}^{k}\}_{k=1}^{K}$ constituting the partition $\breve{\mathcal{P}}$ could be empty; similarly, the set $\mathcal{Z}$ of Lloyd's minima may consist of elements $\breve{\vx}$ for which the $K$ cluster centers $\{\breve{\vx}^{k}\}_{k=1}^{K}$ need not be all distinct. In fact, in general, it is not hard to find initializations that would lead the Lloyd's algorithm to converge to partitions with one or more empty sets or cluster centers that are not all distinct, see~\citep{telgarsky2010hartigan} for more detailed discussions and analyses.
\end{remark}

Finally, denote by $\mathcal{L}_{g}$ the set of global minima of~\eqref{Hmin}. Note that a pair $(\breve{\vx},\breve{\mathcal{P}})\in\mathcal{L}_{g}$ is necessarily a fixed point of Lloyd's algorithm (in the sense of~\eqref{lloyd_up4}).
% otherwise, by instantiating Lloyd's algorithm with  $(\breve{\vx},\breve{\mathcal{P}})$ as the initial cluster center and %cluster estimate, we would obtain a pair $(\acute{\vx},\acute{\mathcal{P}})\not = (\breve{\vx},\breve{\mathcal{P}})$ such that $\mathcal{H}(\acute{\vx},\acute{\mathcal{P}})<\mathcal{H}(\breve{\vx},\breve{\mathcal{P}})$, thus contradicting the hypothesis that $(\breve{\vx},\breve{\mathcal{P}})$ is a global minimizer of~\eqref{Hmin}. This implies that $\mathcal{L}_{g}\subset\mathcal{L}$. Moreover,

Denote by $\mathcal{Z}_{g}$ the set of global minima of the $K$-means formulation~\eqref{Kmeans}, i.e.,
\begin{equation}
\label{def_Z_g}
\mathcal{Z}_{g}=\left\{\breve{\vx}~|~\exists~\breve{\mathcal{P}}~\mbox{such that}~(\breve{\vx},\breve{\mathcal{P}})\in\mathcal{L}_{g}\right\},
\end{equation}
and, note that by~\eqref{lloyd_up5}, $\mathcal{Z}_{g}\subset\mathcal{Z}$.

\subsection{$K$-Means with Distributed Data} \label{sec:dist-framework}
%In many applications of interest data is naturally distributed among a group of agents.
Through the remainder of the paper we will consider the following distributed data framework. Assume there are $M$ agents, each with access to some local dataset $\calD_m$ consisting of points in $\R^p$. In this paper we are interested in studying decentralized methods for computing $K$-means clusterings of the collective dataset
$$
\calD = \calD_1\cup\cdots\cup \calD_M.
$$
We denote by $N_{m}=|\mathcal{D}_{m}|$ the size of the dataset at agent $m$, $m=1,\cdots,M$, with $\sum_{m=1}^{M}N_{m}=N$. For convenience, we will index the $n$-th datapoint, $n=1,\cdots, N_{m}$ in the $m$-th agent by $\mathbf{y}_{m,n}$. We will assume that agents' ability to communicate with one another is restricted as follows.
\begin{assumption} \label{ass:comm-graph}
Agents may only communicate with neighboring agents as defined by some communication graph $G=(V,E)$, where agents correspond to vertices in the graph and an edge (bidirectional) between vertices indicates the ability of the agents to exchange information.
\end{assumption}

Lloyd's algorithm is inherently centralized, in that, both the reassignment and center update steps at each iteration require access to the entire data set $\mathcal{D}$. Implementing Lloyd's algorithm in a multi-agent setting in which the data set $\mathcal{D}$ is distributed across multiple agents, would require either fully centralized coordination or all-to-all communication between the agents at all times (or equivalent assumptions). However, in practice, and especially in large-scale settings, inter-agent communication can be sparse and ad-hoc (for instance, envision a scenario in which the agents or data centers correspond to a network of cellphone users or sensors), centralized coordination may not be achievable and raw data exchange among the agents may be prohibited.

These concerns motivate the current study in which we present efficient distributed approaches for $K$-means clustering in possibly large-scale, realistic multi-agent networks.

%[Maybe remark here that, of course, agents could simply transmit their data to all other agents, but this would be impractical. Maybe also remark here on the difference between this setup and the parallelizable mapreduce type frameworks [cite].]

%Let $\calD = \calD_1\cup\cdots\cup\calD_M$ denote the total dataset. In this paper we are interested in studying decentralized methods for computing $K$-means clusterings of the dataset $\calD$ in this setting.

%Note that in such a setting, when the size of datasets is large it is impractical for agents to share their entire datasets with one another. We are interested in algorithms where agents communicate

%To this end, in Sections ??--?? we will formulate a generalized multi-agent $K$-means objective.
%In Section ?? we begin by formulating an (exact) multi-agent $K$-means objective which is equivalent to the classical centralized objective. In Section ?? we will then consider a relaxed multi-agent $K$-means objective which is amenable to distributed implementation.

\section{A Generalized Multi-Agent $K$-Means Objective} \label{sec:dist_K_means}
As noted in the previous section, Lloyd's algorithm is inherently centralized. In this section we will set up a multi-agent $K$-means objective which will be used to design a class of decentralized $K$-means algorithms.

In Section \ref{sec:objective-formulation} we will formulate the generalized multi-agent $K$-means objective. In Section \ref{sec:gen-min} we will introduce the notion of a generalized Lloyd's minimum (the analog of the Lloyd's minimum in the multi-agent setting) and discuss properties of generalized Lloyd's minima.
%In Section \ref{sec:exact-distributed} we will begin by formulating the $K$-means optimization problem ?? in a multi-agent framework. In Section \ref{sec:relaxed} we then consider a relaxed multi-agent $K$-means objective that is amenable to distributed methods, and will lead to a parametric class of decentralized $K$-means algorithms (see Section \ref{sec:dist-algorithm}).

%In this section we present a generalized multi-agent $k$-means objective which will be used to develop the $NK$-means algorithm. In Section ?? we begin by formulating a multi-agent $K$-means objective which is equivalent to the classical centralized scenario (i.e., an \emph{exact} multi-agent $K$-means objective).
%The exact formulation of the multi-agent $K$-means problem is not readily amenable to distributed algorithms. In Section ?? we present a relaxed multi-agent $K$-means objective which is amenable to distributed approaches and will be used to develop the $NK$-means algorithm in Section ??. In Section ?? we introduce the notion of a generalized multi-agent Lloyd's minimum (or just \emph{generalized minimum}) and discuss basic properties of such points.

%\subsection{A Generalized Multi-Agent $K$-Means Objective}
\label{sec:rel_K_means}

\subsection{Multi-Agent $K$-Means Objective Formulation} \label{sec:objective-formulation}
As a matter of notation, for each agent $m$, denote by $\mathbb{C}_{m}$ the set of all partitions of size $K$ of the local data $\mathcal{D}_{m}$ at agent $m$.
%That is, a $K$-tuple $\mathcal{C}_{m}=\{\mathcal{C}_{m}^{1},\cdots,\mathcal{C}_{m}^{K}\}$ belongs to $\mathbb{C}_{m}$ if and only if $\mathcal{C}_{m}^{k}\subset\mathcal{D}_{m}$ for all $k=1,\cdots,K$, $\mathcal{C}_{m}^{k}\cap\mathcal{C}_{m}^{\acute{k}}=\emptyset$ for $k\neq \acute{k}$ and $\mathcal{C}_{m}^{1}\cup\cdots\cup\mathcal{C}_{m}^{K}=\mathcal{D}_{m}$.

To motivate our construction and in view of the fact that the problem data $\mathcal{D}$ is distributed across multiple agents, we start by noting that the $K$-means formulation~\eqref{Hmin} may be equivalently written as a multi-agent optimization problem in which the goal is to minimize a \emph{separable} cost function subject to an inter-agent coupling constraint as follows:
\begin{equation}
\label{rel1}
\left\{\begin{array}{l}
                                    \inf_{\{(\mathbf{x}_{1},\mathcal{C}_{1}),\cdots,(\mathbf{x}_{M},\mathcal{C}_{M})\}} \sum_{m=1}^{M}\sum_{k=1}^{K}\sum_{\y\in\mathcal{C}_{m}^{k}}\|\y-\mathbf{x}_{m}^{k}\|^{2}\\
                                    \mbox{subject to}\\
                                    \mathbf{x}_{1}=\mathbf{x}_{2}=\cdots=\mathbf{x}_{M},
                                    \end{array}\right.
\end{equation}
where the minimization in~\eqref{rel1} is performed over all \emph{admissible} $M$-tuples of local cluster center and cluster pairs $\{(\mathbf{x}_{1},\mathcal{C}_{1}),\cdots,(\mathbf{x}_{M},\mathcal{C}_{M})\}$, i.e., for each $m=1,\cdots,M$, the pair $(\mathbf{x}_{m},\mathcal{C}_{m})$ consists of a $K$-tuple $\mathbf{x}_{m}=(\mathbf{x}_{m}^{1},\cdots,\mathbf{x}_{m}^{K})$, $\vx_m^k\in \R^p$, $k=1,\ldots,K$ and a partition $\mathcal{C}_{m}\in\mathbb{C}_{m}$.
%of the set $\mathcal{D}_{m}$ corresponding to a clustering of the local data at agent $m$.
At times we will find it convenient to treat the $K$-tuple of cluster centers $\mathbf{x}_{m}=\{\mathbf{x}_{m}^{1},\cdots,\mathbf{x}_{m}^{K}\}$ at agent $m$ as a vector in $\R^{Kp}$ and the $M$-tuple $\vx = \{\vx_1,\ldots,\vx_M\}$ as a vector in $\R^{KMp}$.
%Let $\C = \C_1\times\cdots\times \C_M$ and
%Finally, let $\calK= \{(\vx,\calC): \vx\in\R^{KMp}, \calC \in \C_1\times\cdots\times \C_M\}$ denote the set of all admissible tuples of cluster centers and clusters across agents.

Observe that the formulation \eqref{rel1} has a separable objective in that the objective is the sum of $M$ cost terms in which the $m$-th cost term, $m\in \{1,\cdots,M\}$, is a function of only the local variables $(\mathbf{x}_{m},\mathcal{C}_{m})$ and local data $\mathcal{D}_{m}$ of the $m$-th agent. Note, however, that the equality constraint enforces the coupling between the cluster center variables.

The formulations \eqref{rel1} and \eqref{Hmin} are equivalent in the sense that if $\{(\bx_1,\breve{\calC}_1),\ldots,(\bx_M,\breve{\calC}_M) \}$ is a global minimizer of \eqref{rel1}, then setting
$$
\breve{\mathbf{x}}=\breve{\mathbf{x}}_{1}=\cdots=\breve{\mathbf{x}}_{M} \quad \mbox{ and } \quad \calP^k = \calC_1^k\cup\cdots\cup\calC_M^k
$$
for $k=1,\ldots,K$, and $\calP = \{\calP^1,\ldots,\calP^K\}$, the pair $(\bx,\calP)$ is a global minimizer of \eqref{Hmin}.

In particular, note that $\bx = \{\bx_1,\ldots,\bx_K\}$ is an admissible tuple of cluster centers and $\calP = \{\calP^1,\ldots,\calP^K\}$ is a partition of the collective data $\calD$.
To see that $(\bx,\calP)$ is a global minimizer of \eqref{Hmin} note that if, on the contrary, there exists another pair $(\widehat{\mathbf{x}},\widehat{\mathcal{P}})$ such that $\mathcal{H}(\widehat{\mathbf{x}},\widehat{\mathcal{P}})<\mathcal{H}(\breve{\mathbf{x}},\breve{\mathcal{P}})$, then the $M$-tuple $\{(\widehat{\mathbf{x}},\widehat{\mathcal{C}}_{1}),\cdots,(\widehat{\mathbf{x}},\widehat{\mathcal{C}}_{M})\}$, with $\widehat{\mathcal{C}}^{k}_{m}=\widehat{\mathcal{P}}^{k}\cap\mathcal{D}_{m}$ for all $m=1,\cdots,M$ and $k=1,\cdots,K$, would be feasible for the optimization formulation~\eqref{rel1} and achieve a strictly lower cost than $\{(\breve{\mathbf{x}}_{1},\breve{\mathcal{C}}_{1}),\cdots,(\breve{\mathbf{x}}_{M},\breve{\mathcal{C}}_{M})\}$.
Similarly, it may be shown that if $(\breve{\mathbf{x}},\breve{\mathcal{P}})$ is a global minimizer of~\eqref{Hmin}, then the $M$-tuple $\{(\breve{\mathbf{x}},\breve{\mathcal{C}}_{1}),\cdots,(\breve{\mathbf{x}},\breve{\mathcal{C}}_{M})\}$, with $\breve{\mathcal{C}}^{k}_{m}=\breve{\mathcal{P}}^{k}\cap\mathcal{D}_{m}$ for all $m=1,\cdots,M$ and $k=1,\cdots,K$, constitutes a minimizer of~\eqref{rel1}.

To facilitate the development of iterative distributed $K$-means algorithms, we assume that the agents may exchange information over a preassigned communication graph. We denote by $G=(V,E)$ the inter-agent communication graph with $V=\{1,\cdots,M\}$ denoting the set of $M$ agents and $E$ the set of (undirected) communication links between agents. The inter-agent communication graph is assumed to be  connected, but otherwise arbitrary (possibly sparse). Formally:
\begin{assumption}
\label{ass-connected} The inter-agent communication graph $G$ is connected, or, equivalently, $\lambda_{2}(L)>0$, where $\lambda_{2}(L)$ denotes the second eigenvalue of the graph Laplacian matrix $L$.
\end{assumption}

Under this assumption the formulation~\eqref{rel1} is equivalent to the formulation
\begin{equation}
\label{rel2}
\left\{\begin{array}{l}
                                    \inf_{\{(\mathbf{x}_{1},\mathcal{C}_{1}),\cdots,(\mathbf{x}_{M},\mathcal{C}_{M})\}} \sum_{m=1}^{M}\sum_{k=1}^{K}\sum_{\y\in\mathcal{C}_{m}^{k}}\|\y-\mathbf{x}_{m}^{k}\|^{2}\\
                                    \mbox{subject to}\\
                                    \mathbf{x}_{m}=\mathbf{x}_{l}~\mbox{for all pairs $(m,l)\in E$}.
                                    \end{array}\right.
\end{equation}

%Throughout the remainder of the section, we will consider an unconstrained relaxation of ?? which is amenable to distributed implementation.

%\subsection{The Multi-Agent $K$-Means Problem: Relaxed Formulation} \label{sec:relaxed}
%We now present a generalized multi-agent objective which encodes both the objective in \eqref{rel2} and a relaxation of the constraints in \eqref{rel2}.

The following objective function will allow us to consider unconstrained relaxations of \eqref{rel2}:
For a fixed $\rho\in\N_{+}$ and $\vx\in \R^{KMp}$ and $\calC = (\calC_1,\ldots,\calC_M)\in \C_1\times\cdots\times\C_M$ define
\begin{equation}
\label{rel4}
J^{\rho}(\mathbf{x},\mathcal{C})=\frac{1}{\rho}\sum_{m=1}^{M}\sum_{k=1}^{K}\sum_{\y\in\mathcal{C}_{m}^{k}}\|\y-\mathbf{x}_{m}^{k}\|^{2}+\sum_{(m,l)\in E}\sum_{k=1}^{K}\|\mathbf{x}_{m}^{k}-\mathbf{x}_{l}^{k}\|^{2}.
\end{equation}
The first term above corresponds to the clustering cost while the second term penalizes deviations from the constraint set in \eqref{rel2} with increasing severity as $\rho\to\infty$.
This gives rise to the following relaxation of \eqref{rel2}.
\begin{equation}
\label{rel3}
\inf_{\substack{ ~~\vx\in \R^{KMp}\\ \calC\in\C~~~}} J^\rho(\vx,\calC)
\end{equation}
where $\rho\in\N_{+}$ is a relaxation parameter.

\subsection{Generalized Lloyd's Minimum}  \label{sec:gen-min}

We now introduce the notion of a generalized Lloyd's minimum.
%(We sometimes also refer to this as a generalized minimum of $J^\rho$).
%The notion of a generalized minimum is analogous to the notion of a Lloyd's minimum from Section \ref{sec:k-means-classical}. The relationship between Lloyd's minima and generalized minima of $J^\rho(\cdot,\cdot)$ will be made precise in Theorems \ref{lm:conv_to_Z} and \ref{lm:conv_to_Z}.

We start by noting that for each $\rho\in\N_{+}$, \eqref{rel3} is equivalent to the formulation
\begin{equation}
\label{rel6}
\inf_{\vx\in \R^{KMp}}\myQ^{\rho}(\mathbf{x}) ,
\end{equation}
where
\begin{equation}
\label{rel600}
\myQ^{\rho}(\mathbf{x})= \frac{1}{\rho}\sum_{m=1}^{M}\sum_{\y\in\mathcal{D}_{m}}\min_{k=1,\cdots,K}\|\y-\mathbf{x}_{m}^{k}\|^{2}+\sum_{(m,l)\in E}\sum_{k=1}^{K}\|\mathbf{x}_{m}^{k}-\mathbf{x}_{l}^{k}\|^{2}
\end{equation}
and the minimization in~\eqref{rel6} is to be performed over all $M$-tuples $\mathbf{x}=\{\mathbf{x}_{1},\cdots,\mathbf{x}_{M}\}$ of (local) cluster centers. Indeed, it is straightforward to verify that if $\{(\bx_{1},\calC_{1}),\cdots,(\bx_{M},\calC_{M})\}$ is a global minimum of~\eqref{rel3} then $\{\bx_{1},\cdots,\bx_{M}\}$ is a global minimum of~\eqref{rel6} and, conversely, if $\{\bx_{1},\cdots,\bx_{M}\}$ is a global minimum of~\eqref{rel6}, then the $M$-tuple $\{(\bx_{1},\calC_{1}),\cdots,(\bx_{M},\calC_{M})\}$ is a global minimum of~\eqref{rel3}, where, for each $m$, $\calC_{m}$ may be taken to be an element of $\mathbb{C}_{m}$ satisfying the property that, for each $\y\in\mathcal{D}_{m}$, $\y\in\calC_{m}^{k}$ only if
\begin{equation}
\label{rel7}
\|\y-\bx_{m}^{k}\|\leq\|\y-\bx_{m}^{\acute{k}}\|~~\mbox{for all $\acute{k}\in\{1,\cdots,K\}$}.
\end{equation}
In what follows we will denote by $\mathcal{J}^{\rho}_{g}$ the set of global minima of the formulation~\eqref{rel3}, and denote by $\mathcal{M}^{\rho}_{g}$ the set of global minima of~\eqref{rel6}. We have by the above equivalence,
\begin{equation}
\label{rel8}
\mathcal{M}^{\rho}_{g}=\left\{\bx=\{\bx_{1},\cdots,\bx_{M}\}~|~\exists~\calC=\{\mathcal{C}_{1},\cdots,\mathcal{C}_{M}\}~\mbox{such that}~(\bx,\calC)\in\mathcal{J}^{\rho}_{g}\right\}.
\end{equation}

Now, note that the relaxation~\eqref{rel3} of~\eqref{rel2} is (still) non-convex. The following definition introduces a notion of an approximate minimum.
%We will say that a pair $(\bx,\breve{\calC})$ is a \emph{generalized minimum} of $J(\cdot,\cdot)$ if
%$$
%\bx \in \argmin_{\vx}J^\rho( \vx,\breve{\calC}) \quad \mbox{ and } \quad \breve{\calC} \in \argmin_{\cal C} J^\rho(\bx,\breve{\calC}).
%$$
\begin{definition}[Generalized Lloyd's Minimum]
\label{def:locmin}
A pair $(\bx,\breve{\mathcal{C}})\in\calK$ is said to be a generalized Lloyd's minimum or a generalized minimum of $J^{\rho}(\cdot,\cdot)$ if\\
\noindent  (i) for each $m$ and $\y\in\mathcal{D}_{m}$,
\begin{equation}
\label{def:locmin2}
\mbox{$\y\in\calC_{m}^{k}$ only if $\|\y-\bx_{m}^{k}\|\leq\|\y-\bx_{m}^{\acute{k}}\|,~\forall \acute{k}$}
\end{equation}
and,\\
\noindent (ii) for each $m$ and $k$,
\begin{equation}
\label{def:locmin3}
\bx_{m}^{k}=\frac{(1/\rho)\sum_{\y\in\calC_{m}^{k}}\y+\sum_{l\in\Omega_{m}}\bx_{l}^{k}}{(1/\rho)|\calC_{m}^{k}|+|\Omega_{m}|}.
\end{equation}
\end{definition}

\begin{remark} \label{remark:gen-min-vs-lloyds}
Note that given a fixed clustering $\calC\in \C_1\times\cdots\times\C_M$, the function $\vx \mapsto J^\rho(\vx,\calC)$ is convex and differentiable with a unique minimizer. In particular, a vector $\vx^\star\in\R^{MKp}$ minimizes $\vx \mapsto J^\rho(\vx,\calC)$ if and only if
$$
\vx_{m}^{\star,k}=\frac{(1/\rho)\sum_{\y\in\calC_{m}^{k}}\y+\sum_{l\in\Omega_{m}}\vx_{l}^{\star, k}}{(1/\rho)|\calC_{m}^{k}|+|\Omega_{m}|}
$$
holds for all $k=1,\ldots,K$, $m=1,\ldots,M$.
Thus, in words, the above definition states that a tuple $(\bx,\breve{\calC})$ is an element of $\calJ^\rho$ if and only if (i) for $\bx$ fixed,  $\breve{\calC}$ is an ``optimal'' partitioning of the datapoints, and (ii), for $\breve{\calC}$ fixed, $\bx$ optimizes the generalized objective $\vx \mapsto J^\rho(\vx,\breve{\calC})$. This may be compared to the classical definition of a Lloyd's minimum \eqref{lloyd_up4}.

\end{remark}

The following notation, which will be used through the remainder of the paper, will facilitate discussion of generalized minima. For each $\mathbf{x}\in\mathbb{R}^{MKp}$, i.e., an $M$-tuple $\mathbf{x} = \{\mathbf{x}_{1},\cdots,\mathbf{x}_{M}\}$ of potential cluster centers,
%let
%\begin{align}
%\calU_{\vx} = \Big\{\calC \in \C: \y\in\mathcal{C}_{m}^{k}~\Longrightarrow~\|\y-\mathbf{x}_{m}^{k}\|\leq\|\y-\mathbf{x}_{m}^{\acute{k}}\|,\\ \forall \acute{k}=1,\ldots,K,~ m=1,\cdots,M, ~n=1,\cdots,N_m \Big\}
%\end{align}
let $\mathcal{U}_{\mathbf{x}}$ denote the subset of $\mathbb{C}_{1}\times\cdots\times\mathbb{C}_{M}$ such that a $M$-tuple $\mathcal{C}\in \mathbb{C}_{1}\times\cdots\times\mathbb{C}_{M}$ belongs to $\mathcal{U}_{\mathbf{x}}$ if the following holds for all $m=1,\cdots,M$ and $\mathbf{y}\in\mathcal{D}_{m}$:
%\begin{equation}
%\label{def_U}
%\bigcup_{k=1}^{K}\mathcal{C}_{m}^{k}=\mathcal{D}_{m}~~\mbox{and}~~\mathcal{C}_{m}^{k}\cap\mathcal{C}_{m}^{\acute{k}}=\emptyset~\mbox{if}~k\neq \acute{k},
%\end{equation}
%and $\forall n=1,\cdots,N_{m}$,
\begin{equation}
\label{def_U1}
\y\in\mathcal{C}_{m}^{k}~\Longrightarrow~\|\y-\mathbf{x}_{m}^{k}\|\leq\|\y-\mathbf{x}_{m}^{\acute{k}}\|,~\forall \acute{k}.
\end{equation}

Note that, analogous to the conditions given in \eqref{lloyd_up4} for a Lloyd's minimum, a tuple $(\bx,\bcalC)$ is a generalized minimum if and only if
$$
\bcalC \in \calU_{\bx} \quad \mbox{ and } \quad \bx \in \argmin_{\vx} J^\rho(\vx,\bcalC).
$$

The set of all generalized minima of $J^{\rho}(\cdot,\cdot)$ will be denoted by $\mathcal{J}^{\rho}$. Additionally, denote by $\mathcal{M}^{\rho}$ the set
\begin{equation}
\label{rel9}
\mathcal{M}^{\rho}=\left\{\bx=\{\bx_{1},\cdots,\bx_{M}\}~|~\exists~\calC=\{\mathcal{C}_{1},\cdots,\mathcal{C}_{M}\}~\mbox{such that}~(\bx,\calC)\in\mathcal{J}^{\rho}\right\}.
\end{equation}
By abusing notation, an element $\bx\in\mathcal{M}^{\rho}$ will also be referred to as a generalized minimum of $J^{\rho}(\cdot,\cdot)$ or $\myQ^\rho(\cdot)$.
%the formulation~\eqref{rel6}.

For each $\vx\in\R^{KMp}$ define $\overline{\mathcal{U}}_{\mathbf{x}}$ (may be empty) to be the subset of $\calU_{\vx}$ given by
$$
\overline{\calU}_{\vx} = \bigg\{ \calC\in\calU_{\vx}:\,  \mathbf{x}_{m}^{k}=\frac{(1/\rho)\sum_{\y\in\mathcal{C}_{m}^{k}}\y+\sum_{l\in\Omega_{m}}\mathbf{x}_{l}^{k}}{(1/\rho)|\mathcal{C}_{m}^{k}|+|\Omega_{m}|}, ~~\forall m \mbox{ and } k \bigg\}.
$$
%the subset of $\mathcal{U}_{\mathbf{x}}$ which consists of elements $\mathcal{C}$ of $\mathcal{U}_{\mathbf{x}}$ which satisfy the additional condition that for all $m$ and $k$,
%\begin{equation}
%\label{def_U2}
%\mathbf{x}_{m}^{k}=\frac{(1/\rho)\sum_{\y\in\mathcal{C}_{m}^{k}}\y+\sum_{l\in\Omega_{m}}\mathbf{x}_{l}^{k}}{(1/\rho)|\mathcal{C}_{m}^{k}|+|\Omega_{m}|}.
%\end{equation}
Note that a $M$-tuple $\mathbf{x}=\{\mathbf{x}_{1},\cdots,\mathbf{x}_{M}\}$ is a generalized minimum of $J^{\rho}(\cdot,\cdot)$ if and only if $\overline{\mathcal{U}}_{\mathbf{x}}$ is non-empty, i.e.,
\begin{equation} \label{eq:calU-condition}
\vx \in \calM^\rho \iff \overline{\calU}_{\vx} \not= \emptyset.
\end{equation}

The following proposition shows that the set of generalized minima (in the sense of Definition~\ref{def:locmin}) subsumes the set of global minima of $J^{\rho}(\cdot,\cdot)$.
\begin{proposition}
\label{prop:rel_glob_loc}
For each $\rho\in\N_{+}$, we have that $\mathcal{J}^{\rho}_{g}\subset\mathcal{J}^{\rho}$ and $\mathcal{M}^{\rho}_{g}\subset\mathcal{M}^{\rho}$.
\end{proposition}

The proof of this proposition follows readily from Remark \ref{remark:gen-min-vs-lloyds}.

In Sections \ref{sec:genminimaconv} and \ref{sec:global-min-convergence} (and, in particular, Theorems \ref{lm:conv_to_Z} and \ref{lm:gloptconv}) we will formally quantify the relationship between the set $\mathcal{M}^{\rho}$ of generalized minima of $J^{\rho}(\cdot,\cdot)$ and the set $\mathcal{Z}$ of Lloyd's minima of the classical $K$-means formulation given in Section \ref{Kmeans}. Informally, we will show that as $\rho\rightarrow\infty$ the set $\mathcal{M}^{\rho}$ \emph{approaches} the set of Lloyd's minima $\mathcal{Z}$, thus justifying the role of $\rho$ as a relaxation parameter.

\section{Distributed $NK$-means Algorithm} \label{sec:dist-algorithm}
%As with \eqref{costH} and the classical (centralized) Lloyd's algorithm,
%the optimization problem \eqref{rel3} naturally gives rise to a (distributed) $K$-means algorithm as follows.
%We now propose a distributed $K$-means clustering algorithm.

We now propose a distributed $K$-means clustering algorithm, designated as the $NK$-means algorithm. The $NK$-means is a distributed iterative algorithm, in which each agent $m$ updates its local estimate of the cluster centers by simultaneously processing their local data $\mathcal{D}_{m}$ and information received from neighboring agents (to be specified shortly). For each agent $m=1,\ldots,M$, let $\vx_m(t)= \{\vx_m^1(t),\ldots,\vx_m^K(t)\}$ denote a set of (local) cluster centers and let $\calC_m(t) = \{\calC_m^1(t),\ldots,\calC_m^K(t)\}\in\C_m$ denote a $K$-partition of the local dataset $\calD_m$ at iteration $t\in \N_{+}$.

%Denote by $\mathcal{U}_{m}$ the collection of all viable partitions of $\calD_m$, i.e.,
%\begin{equation}
%\label{cluster-set-m}
%\calU_m = \bigg\{(\calC_m^1,\ldots,\calC_m^K):~ \bigcup_{k=1}^{K}\mathcal{C}_{m}^{k}=\mathcal{D}_{m}~~\mbox{and}~~\mathcal{C}_{m}^{k}\cap\mathcal{C}_{m}^{\acute{k}}=\emptyset~\mbox{if}~k\neq \acute{k} \bigg\}
%\end{equation}

In the following algorithm, each agent will update its sequence $\{\mathbf{x}_{m}(t),\mathcal{C}_{m}(t)\}_{t\in \N_{+}}$ of (local) cluster centers and clusters in a distributed fashion. More precisely, the algorithm is initialized by letting each agent $m\in \{1,\ldots,M\}$ select an arbitrary initial seed denoted by $\vx_m(0) \in \R^{Mp}$. Subsequently, at each iteration $t\geq 0$,
%given $\vx_m(t-1)$ and $\calC_m(t)$
each agent $m=1,\ldots,M$ performs a \emph{reassign} step and a \emph{center update} step to refine its current local partition $\calC_m(t)$ and cluster centers estimates $\vx_m(t)$ as follows:\\
\begin{enumerate}[(i)]
\item \noindent{\bf Reassign}. The local partition $\mathcal{C}_{m}(t+1)$ at agent $m$ is taken to be an arbitrary partition of $\calD_m$ satisfying
\begin{equation}
\label{reassign1}
\y\in\mathcal{C}_{m}^{k}(t+1)~\Longrightarrow~\|\y-\mathbf{x}_{m}^{k}(t)\|\leq\|\y-\mathbf{x}_{m}^{\acute{k}}(t)\|,~\forall \acute{k},
\end{equation}
for each $\vy\in \calD_m$.\\
\item \noindent{\bf Center update}. Once $\mathcal{C}_{m}(t+1)$ is obtained (selected), for each $k=1,\cdots,K$, the $k$-th cluster center at agent $m$ is updated as
\begin{equation}
\label{cent_up}
\mathbf{x}_{m}^{k}(t+1)=\mathbf{x}_{m}^{k}(t)-\alpha\left(\mathbf{x}_{m}^{k}(t)-\bmu_{m}^{k}(t+1)\right),
\end{equation}
where
\begin{equation}
\label{cent_up1}
\bmu_{m}^{k}(t+1)=\frac{(1/\rho)\sum_{\y\in\mathcal{C}_{m}^{k}(t+1)}\y+\sum_{l\in\Omega_{m}}\mathbf{x}_{l}^{k}(t)}{(1/\rho)|\mathcal{C}_{k}^{m}(t+1)|+|\Omega_{m}|}
\end{equation}
and $\Omega_{m}$ denotes the communication neighborhood of agent $m$ (with respect to the assigned communication graph $G$).
\end{enumerate}

We note that the parameter $\rho\in (0,\infty)$ is a design constant which affects the quality of the solution asymptotically obtained (see Theorem \ref{lm:gloptconv}), whereas, the constant $\alpha$ is a positive weight parameter assumed to satisfy the following assumption.
\begin{assumption}
\label{ass-alpha} The weight parameter $\alpha$ satisfies the following condition:
\begin{equation}
0<\alpha<\frac{d_{\min}}{(1/\rho)N^{\ast}+\lambda_{M}(L)},
\end{equation}
where $N^{\ast}=\max(N_{1},\cdots,N_{M})$, $d_{\min}$ and $d_{\max}$ denote the minimum and maximum degrees respectively of the inter-agent communication graph, and $\lambda_{M}(L)$ is the largest eigenvalue of the graph Laplacian matrix $L$.
\end{assumption}

We remark that Assumption~\ref{ass-connected} implies that $|\Omega_{m}|>0$ for each $m$.
%, which, in turn, implies that the center update step~\eqref{cent_up}--\eqref{cent_up1} is well-defined.\footnote{Note, BS: I think the update steps above would still be well defined, they just rely only on local information.}
The procedure~\eqref{reassign1}--\eqref{cent_up1} is clearly distributed as at any given instant $t$ the cluster and cluster center update at an agent $m$ is based on purely local computation and information exchange with neighboring agents. In particular, we  note that the agents do not exchange raw data, i.e., individual datasets are not exchanged, but achieve collaboration by means of sharing their local cluster center estimates with neighboring agents. We will refer to the distributed algorithm~\eqref{reassign1}--\eqref{cent_up1} as \emph{networked $K$-means} or $NK$-means in short. Note that by varying $\rho$ in the interval $(0,\infty)$, we, in fact, obtain a parametric family of algorithms.

The rest of the paper is devoted to the convergence analysis of the $NK$-means algorithm and quantifying how it relates to the original (centralized) $K$-means objective~\eqref{Kmeans}.
%To this end, first, for each fixed $\rho\in (0,\infty)$, we will study the convergence of the NK-means (achieved in Section~\ref{sec:NKmeansconv}). We will show that at each agent $m$ the tuple $\{\mathbf{x}_{m}(t),\mathcal{C}_{m}(t)\}$ converges to a limit $(\bx_{m},\brC_{m})$ as $t\rightarrow\infty$. We will show that the limit tuples, which are functions of the parameter $\rho$, are related to a certain optimization problem (to be introduced in Section~\ref{sec:rel_K_means}). Specifically, in Section~\ref{sec:rel_K_means} we will introduce a parametric family of optimization formulations, the respective optimization objectives (to be referred to as \emph{generalized multi-agent $K$-means objectives}) being indexed by the parameter $\rho$; intuitively, each of these objectives may be viewed as a certain \emph{relaxation} of the actual $K$-means objective with relaxation parameter $\rho$. [COMMENT ON SECTIONS 5 AND 6?] Subsequently, in Section~\ref{sec:genminimaconv} we will study the relationship between these parametric family of generalized multi-agent $K$-means objectives and the actual $K$-means objective~\eqref{Kmeans} as a function of $\rho$.\footnote{Note, BS: Comment on what the relationship actually is. I.e., something like minima of the relaxed problem converge to minima of the original problem.} Finally, the results in Section~\ref{sec:NKmeansconv} and Section~\ref{sec:genminimaconv} will be used to quantify the relationship between the limiting points of the $NK$-means algorithm and the Lloyd's minima of the $K$-means.

\section{Main Results}
\label{sec:main_res}
We now collect the main results of the paper (Theorems \ref{lm:conv_to_Z}--\ref{th:conv} below).

Our first main result characterizes the relationship between classical Lloyd's minima  and generalized Lloyd's minima (with parameter $\rho$) introduced in Section \ref{sec:gen-min}. In particular, the result shows that, as $\rho\to\infty$, the set of generalized Lloyd's minima converges to the set of classical Lloyd's minima.

Before stating the theorem, we require a few definitions. Note that it is possible to have a pair $(\vx,\calP)$ that is optimal in the sense of Lloyd \eqref{lloyd_up4}, but where some cluster $\calP^k$ in the partition $\calP$ is empty and the corresponding cluster center $\vx^k$ is not contained in $\cco(\calD)$. Such a solution may be considered to be degenerate. In the multi-agent setting, a similar issue can occur with generalized minima in the sense of Definition \ref{def:locmin}. In order to avoid such degeneracies, it is helpful to consider solutions whose cluster centers are contained in $\cco(\calD)$. Formally, let
\begin{align}
\label{lm:conv_to_Z_2}
\OZ=\left\{\vx\in\mathcal{Z}~:~\vx^{k}\in\cco(\mathcal{D})~~\forall k\right\},
\end{align}
and for $\rho\in\N_+$ let
\begin{equation}
\label{corr:cohull1}
\oM^{\rho}=\left\{\bx\in\mathcal{M}^{\rho}~:~\bx_{m}^{k}\in\cco(\mathcal{D})~~\forall m,k\right\}.
\end{equation}

Our first main result is the following.
\begin{theorem}
\label{lm:conv_to_Z} Let Assumptions \ref{ass:data_size}, \ref{ass:comm-graph}, \ref{ass-connected}, and \ref{ass-alpha} hold and $\{\bx^{\rho}\}_{\rho\in\N_{+}}$ be a sequence such that $\bx^{\rho}\in\oM^{\rho}$ for each $\rho\in\N_{+}$. Then, for each $m=1,\ldots,M$, we have that
\begin{equation}
\label{lm:conv_to_Z_1}
\lim_{\rho\rightarrow\infty}d\left(\bx_{m}^{\rho},\OZ\right)=0.
\end{equation}
\end{theorem}

Theorem \ref{lm:conv_to_Z} relates the generalized minima (as a function of the parameter $\rho$) to Lloyd's minima of the (centralized) clustering problem as $\rho\rightarrow\infty$. Specifically, it states that for generalized minima $\bx^{\rho}=\{\bx^{\rho}_{1},\cdots,\bx^{\rho}_{M}\}$, each component $\bx^{\rho}_{m}=\{\bx_{m}^{\rho,1},\cdots,\bx_{m}^{\rho,K}\}$ (corresponding to a potential $K$-means clustering) approaches the set $\OZ$ of Lloyd's minima of the (centralized) clustering problem as $\rho\rightarrow\infty$. It is important to note here that the convergence in Theorem \ref{lm:conv_to_Z} holds irrespective of how the sequence $\{\bx^{\rho}\}_{\rho\in\N_{+}}$ is constructed and, in particular, independent of any algorithm that might be used to generate such a sequence. In fact, a stronger version of Theorem \ref{lm:conv_to_Z} is obtained in Corollary~\ref{corr:conv_to_Z}, where we show that the convergence in~\eqref{lm:conv_to_Z_1} is \emph{uniform}, i.e., the following holds:
$$\lim_{\rho\rightarrow\infty}\sup_{\bx^{\rho}\in\oM^{\rho}}\max_{m=1,\cdots,M}d\left(\bx_{m}^{\rho},\OZ\right)=0.$$
Theorem \ref{lm:conv_to_Z} is proved in Section \ref{sec:genminimaconv} where we consider asymptotic properties of the set of generalized minima $\calM^\rho$ as $\rho\to\infty$.

Our second main result characterizes the relationship between global minima of the classical $K$-means clustering objective and global minima of the mutli-agent objective introduced in Section \ref{sec:dist_K_means}. First, the result shows that as $\rho\to\infty$, global minima of the multi-agent objective $\myQ^\rho$ converge to global minima of the classical $K$-means objective. Second, the result quantifies the efficiency of global minima of $\myQ^\rho$ in terms of the classical $K$-means objective. This can facilitate the choice of $\rho$ in practice and gives a characterization of the rate at which generalized Lloyd's minima converge to the set of Lloyd's minima as $\rho\to\infty$.

\begin{theorem}
\label{lm:gloptconv}
Let Assumptions \ref{ass:data_size}, \ref{ass:comm-graph}, \ref{ass-connected}, and \ref{ass-alpha} hold and $\{\bx^{\rho}\}_{\rho\in\N_{+}}$ be a sequence such that $\bx^{\rho}\in\mathcal{M}_{g}^{\rho}$ for each $\rho\in\N_{+}$. Then, for each $m$, we have that
\begin{equation}
\label{lm:gloptconv1}
\lim_{\rho\rightarrow\infty}d\left(\bx_{m}^{\rho},\mathcal{Z}_{g}\right)=0.
\end{equation}
%where $\bx_{m}^{\rho}$ denotes the $K$-tuple $\{\bx_{m}^{\rho,1},\cdots,\bx_{m}^{\rho,K}\}$ of cluster center estimates at agent $m$.
Furthermore, letting $\mathcal{F}^{\ast}$ denote the global minimum value of~\eqref{Kmeans}
%Let Assumption~\eqref{ass:data_size} and Assumption~\eqref{ass-connected} hold and $\bx$ be a global minimizer of $\myQ^{\rho}(\cdot)$, i.e., $\bx\in\mathcal{M}^{\rho}_{g}$. For each $m$, denote by $\bx_{m}$ the $K$-tuple $\{\bx_{m}^{1},\cdots,\bx_{m}^{K}\}$. Then,
we have that
\begin{equation}
\label{lm:gloptgap1a}
\mathcal{F}(\bx_{m}^\rho)\leq\mathcal{F}^{\ast} + \frac{16\sqrt{M}R^{2}_{0}|\mathcal{D}|^{2}}{\rho\lambda_{2}(L)},
\end{equation}
where $R_{0}=\max_{\mathbf{v}\in\cco(\mathcal{D})}\|\mathbf{v}\|$.
\end{theorem}
Similar to Theorem~\ref{lm:conv_to_Z}, the convergence in Theorem~\ref{lm:gloptconv} may be shown to be uniform, i.e., the following holds (see Corollary~\ref{corr:gloptconv})
$$\lim_{\rho\rightarrow\infty}\sup_{\bx^{\rho}\in\mathcal{M}_{g}^{\rho}}\max_{m=1,\cdots,M}d\left(\bx_{m}^{\rho},\mathcal{Z}_{g}\right)=0.$$
Theorem \ref{lm:gloptconv} is proved in Section \ref{sec:global-min-convergence} where we consider asymptotic properties of the set of global minima $\calM^\rho_g$ as $\rho\to\infty$.

Our final main result shows that the $NK$-means algorithm (with parameter $\rho$) converges to the set of generalized Lloyd's minima.
In what follows we will use the notation $\vx_t$ to denote the vector $\vecc_{m,k}(\vx_m^k(t))$ at iteration $t$ and we let $\calC_t=\{\calC_1(t),\ldots,\calC_M(t)\}$ denote the tuple containing the clusterings at each agent.
\begin{theorem}
\label{th:conv} For a fixed $\rho$, let Assumptions \ref{ass:data_size}, \ref{ass:comm-graph}, \ref{ass-connected}, and \ref{ass-alpha} hold and let $\{\mathbf{x}_{t}\}$ be a sequence of cluster centers generated by the $NK$-means algorithm, i.e., the distributed procedure~\eqref{reassign1}--\eqref{cent_up1}. Then $\{\mathbf{x}_{t}\}$ converges as $t\rightarrow\infty$ to a generalized minimum $\bx$ of $\myQ^{\rho}(\cdot)$.
Furthermore the sequence of clusterings $\{\calC_{t}\}$ generated by \eqref{reassign1}--\eqref{cent_up1} converges in finite time to the set of optimal clusterings under $\bx$, i.e., there exists a finite $T>0$ such that
$\calC_t \in \calU_{\bx}$ for all $t\geq T$.
\end{theorem}

Theorem \ref{th:conv} is proved in Section \ref{sec:NKmeansconv}  where we consider convergence properties of the $NK$-means algorithm.
Combined with Theorem \ref{lm:conv_to_Z} (see also Corollary~\ref{corr:conv_to_Z}), this shows that the $NK$-means algorithm can be used to compute Lloyd's minima up to arbitrary accuracy. Various aspects of the convergence of the $NK$-means algorithm are discussed in Section~\ref{sec:main-res-discussion} below and further illustrated through a numerical study in Section~\ref{sec:sims}.

\subsection{Discussion} \label{sec:main-res-discussion}
%\noindent \textbf{Convergence of cluster heads and Partitions}.
\noindent \textbf{Convergence of the $NK$-means Algorithm}.
Theorem \ref{th:conv} ensures that the sequence of cluster heads $\{\bx_t\}$ generated by the $NK$-means algorithm converges to some tuple $\bx\in\R^{MKp}$ that is optimal in the sense of a generalized Lloyd's minimum.

Likewise, the sequence of clusters $\calC_t$ generated by the $NK$-means algorithm also converges, but to the \emph{set} of optimal clusterings corresponding to $\bx$. To be more precise, a technical (and generally pathalogical) difficulty which arises in studying $K$-means clustering is the problem that an optimal tuple of cluster heads may admit more than one permissible clustering. In particular, if some datapoint is precisely equidistant from two or more cluster heads, then the datapoint may be optimally assigned to any cluster corresponding to any such cluster head.

Assuming $\vx_t \to \bx$, recall that $\calU_{\bx}$ is the set of permissible clusterings corresponding to the (optimal) limit point $\bx$ (see \eqref{def_U1} and preceeding discussion). If $\calU_{\bx}$ contains only one permissible clustering, then $\calC_t$ will converge to the (unique) optimal clustering in $\calU_{\bx}$ after some finite time $T$ (see Lemma \ref{lm:limpc} and Remark \ref{remark:partition-convergence}). In the degenerate case that $\calU_{\bx}$ contains more than one permissible clustering (i.e., if for some $m$ there exist two or more cluster heads $\bx_m^k$ in $\bx_m$ which are equidistant from a datapoint in $\calD_m$), then $\calC_t$ will converge to the \emph{set} of optimal clusterings $\calU_{\bx}$ after some finite time $T$, but may continue to oscillate between optimal clusterings in $\calU_{\bx}$ infinitely often. Since the clusterings obtained are nonetheless optimal after time $T$, in an abuse of terminology we say that $\calC_t$ has converged after time $T$.
At each time step $t$, the clustering $\calC_t$ uniquely generates a partition $\calP_t$ of the joint dataset $\calD$. That is, $\calP_t = (\calP_t^1,\ldots,\calP_t^K)$ is the partition of $\calD$ generated by $\calC_t$, where $\calP_t^k = \calC_1^k\cup\cdots\cup\calC_M^k$, $k=1,\ldots,K$. The implication of the above convergence result is that after some finite time $T$, $\calP_t$ converges to a partition of the collective dataset corresponding to a generalized Lloyd's minimum with parameter $\rho$.

We emphasize that the convergence of the cluster heads $\vx_t$ occurs over an infinite time horizon, while the convergence of the clusters $\calC_t$ occurs in finite time. However, the rate of convergence of the cluster heads $\vx_t$ is generally exponential. In particular, note that if $\calC_t$ converges in finite time, then dynamics \eqref{reassign1}--\eqref{cent_up1} are linear after time $T$ and $\vx_t$ converges exponentially to $\bx$.

%\subsection{Discussion}
%[ADD DISCUSSION ABOUT PARTITIONS CONVERGING IN FINITE TIME BUT cluster headS CONVERGING OVER INFINITE TIME HORIZON BUT AT EXPONENTIAL RATE. (Add some remark about this to proofs somewhere.)]

$~$\\
\noindent \textbf{Tradeoffs with Parameter $\rho$}.
The above results show that limit points of the $NK$-means algorithm may be brought arbitrarily close to the set of Lloyd's minima by choosing $\rho$ sufficiently large. However, there is an inherent tradeoff in choosing $\rho$ large. Loosely speaking, as $\rho\to\infty$, the dynamics \eqref{reassign1}--\eqref{cent_up1} place greater relative weight on terms from the consensus component of the generalized multi-agent objective and less weight on $K$-means components. In practice, as $\rho\to\infty$ this results in fast dynamics orthogonal to (i.e., towards) the consensus subspace, and slower dynamics tangential to the consensus subspace. (This may be formalized by considering the dynamics of the mean process $\bar \vx^k(t) = \frac{1}{M} \sum_{m=1}^M \vx_m^k(t)$, $k=1,\ldots,K$.) Overall, this means that large values of $\rho$ improve the quality of the limit points of the $NK$-means algorithm, but can result in slower convergence.

\bigskip
\noindent \textbf{Finite $\rho$ convergence of Clusters}.
In the above discussion we observed that as $t\to\infty$ (with $\rho$ fixed), the cluster heads generated by the $NK$-means algorithm converge over an infinite time horizon, while the clusters converge in finite time. A similar property holds for the clusters obtained at generalized minima as $\rho\to\infty$.

Theorem \ref{lm:conv_to_Z} showed that, as $\rho\to\infty$ the cluster heads obtained at generalized minima with parameter $\rho$ converge asymptotically to the set of Lloyd’s minima. As a corollary to Theorem \ref{lm:conv_to_Z} we may show that there exists some finite $\bar \rho$ such that for all $\rho\geq \bar \rho$, the clustering obtained at any generalized minimum with parameter $\rho$ is cost-equivalent to a Lloyd’s minimum.
This is formalized in the following result.

\begin{corollary}[Finite $\rho$ convergence of clusters] \label{cor:finite-rho}
Let Assumptions \ref{ass:data_size}, \ref{ass:comm-graph}, and \ref{ass-connected} hold. There exists a $\bar\rho >0$ such that for any $\rho \geq \bar \rho$ the following holds: Suppose $(\bx^\rho,\breve\calC^\rho) \in \calJ^\rho$, with $\bx^\rho\in \overline{\calM}^{\rho}$. Let $\breve\calP^\rho = (\breve\calP^{\rho,1},\ldots,\breve\calP^{\rho,K})$ be the partition of $\calD$ naturally generated by $\breve\calC^\rho$ as
\begin{equation} \label{eq:generated-partition}
\breve \calP^{\rho,k} = \breve\calC_1^{\rho,k}\cup\cdots\cup\breve\calC_M^{\rho,k},\quad k=1,\ldots,K.
\end{equation}
Then the partition $\breve\calP^\rho$ is cost equivalent to a classical Lloyd's minimum. That is, there exists a Lloyd's minimum pair $(\vx,\calP)\in \calL$ such that $\calH(\vx,\calP) = \calH(\vx,\breve\calP^{\rho}) $.
%where  $\bar \vx = (\bar \vx^{1},\ldots,\bar \vx^{K})$, and $\bar \vx^{k} = |\breve \calP^{\rho,k}|^{-1}\sum_{\vy\in\breve \calP^{\rho,k}} \vy$ is the centroid of $\breve \calP^{\rho,k}$.
\end{corollary}

A proof of this corollary can be found in Section \ref{sec:finite-rho-conv}.

%The rough idea of the proof is as follows. Theorem ?? shows that any generalized Lloyd's minimum $\bx^\rho$ of $Q^\rho(\cdot)$ converges asymptotically to the set of classical Lloyd's minima as $\rho\to\infty$. Once $\bx^\rho$ is sufficiently close to the set of classical Lloyd's minima (so, for all $\rho$ sufficiently large), the clusterings converge to the set of optimal clusterings.

\bigskip
\noindent \textbf{Computing Optimal Clusterings.}
Many of our main results are asymptotic in nature. However, we emphasize that this \emph{does not} imply that the $NK$-means algorithm can only be used to \emph{approximate} an optimal $K$-means clustering.

We emphasize the following key implication of our results: For all values of $\rho$ sufficiently large, the $NK$-means algorithm will converge in finite-time to a $K$-means clustering that is cost-equivalent to a Lloyd's minimum. This follows from Theorem \ref{th:conv} and Corollary \ref{cor:finite-rho}.

The technicality in the above statement (convergence to a \emph{cost-equivalent} Lloyd's minimum clustering) arises due to the weak assumptions we make on the dataset. In particular, we allow for $\calD$ to possess repeated datapoints. As noted in Section \ref{introduction}, this is necessary due to the multi-agent nature of the setup: e.g., two agents may collect some datapoints from the same datasource or datastream. This introduces technical challenges into the analysis and leads to slightly weaker convergence results (e.g., convergence to clusterings that are cost-equivalent to Lloyd's minima).
%In practice, we often expect convergence to genuine Lloyd's minima partitions.
However, for practical purposes the clusterings obtained by the $NK$-means algorithm are functionally equivalent to Lloyd's minima.

\section{Illustrative Example} \label{sec:sims}
We now present a simple example illustrating the operation and salient features of the $NK$-means algorithm. Consider a system with 10 agents, let $p=1$ so data points reside in $\R$, and let the data set $D_m$ for agent $m=1,\ldots,10$ be generated by drawing 50 independent samples from the normal distribution $\calN(\mu_m,\sigma_m^2)$, with mean $\mu_m$ and variance $\sigma_m^2$.
%We set the means as $(\mu_1,\ldots,\mu_10) = (5, 20, 30, 60, 100, 5, 20, 30, 60, 100)$, and set the variance to be one for all agents, i.e., $\sigma_1 = \cdots = \sigma_5 = 1$.
Note that the full dataset $\calD= \bigcup_{m=1}^{10} \calD_m$ is drawn from a Gaussian mixture of the distributions $\calN(\mu_m, \sigma_m^2)$, $m=1,\ldots,10$.

We note that the behavior of classical $K$-means algorithm has been well studied for such Gaussian mixture models and is known to be consistent in the sense of \citep{pollard1981strong}. We will not give an in-depth treatment of issues of statistical consistency here.

%that, if $K$ is set to be the number of distributions in the mixture, then as the number of sampled datapoints goes to infinity, the set of cluster heads generated by the $K$-means algorithm converges to the parametric means of the mixture, almost surely. (See [CITE] for more details.)

For our simulation example we let $\mu = (\mu_m)_{m=1}^{10} = (5, 20, 30, 60, 100, 5, 20, 30, 60, 100)$ and let $\sigma_m^2 = 1$ for all $m$. Note that, in this case, the full dataset is sampled from a uniform mixture of 5 Gaussian distributions ($\mu$ consists of 5 distinct means repeated twice). Let the number of clusters be set to $K=5$ and let the graph $G$ be given by a ring graph. We let each agent's initial estimate of the cluster heads, $x_m(0) = (x_m^k(0))_{k=1}^5$ be  given by $x_m = (0, 20, 40, 60, 80)$.

%generated by drawing each $x_i^k(1)$ uniformly at random over the interval $[0,100]$.

Figure \ref{fig:plot1} plots the cluster head estimates for each agent over time for one instance of the $NK$-means algorithm with the above initialization scheme for parameters $\rho = 2,5,10,10^2,10^3,10^4$.
For $\rho$ small ($\rho = 2,5$), consensus is weakly enforced; the cluster heads of the individual agents are asymptotically separated and the algorithm performs relatively poorly. For $\rho \geq 10$ the vector of cluster head estimates $\vx_m(t) = (\vx_m^k(t))_{k=1}^5$ approximates well the vector of sample means of the Gaussian mixture.

In general, the partitions generated by the $NK$-means algorithm will converge in finite time (see Section \ref{sec:main-res-discussion}). The convergence time for the partitions for each value of $\rho$ for this example is given in Table \ref{table:tab1}.
In contrast to the convergence of the partitions, the convergence of the \emph{cluster heads} $(\vx_m^k(t))$ occurs over an infinite time horizon. Using the linearity of the update equations \eqref{reassign1}--\eqref{cent_up1} it is straightforward to show that once the partitions have converged, that rate of convergence of the cluster heads is exponential. However, as noted in Section \ref{sec:main_res}, there is an inherent tradeoff in the choice of $\rho$ in terms of the rate of convergence. In particular, increasing $\rho$ improves the quality of limit points of the $NK$-means algorithm (in terms of distance to Lloyd's minima) and increases the rate at which agents approach consensus. However, increasing $\rho$ results in slower dynamics tangential to the consensus subspace, and reduces the overall rate of convergence. %This can be see in Figure \ref{fig:plot1}.

Table \ref{table:tab1} shows that the partitions converge most quickly for $\rho=100$. In Figure \ref{fig:plot1} we see that as $\rho$ increases, the cluster heads approach consensus more quickly, however, the asymptotic rate of convergence towards the limiting cluster head tuple decreases, which comports with the above observations.
%NOTE: The Balcan papers look at the communication cost of these algorithms. This is roughly equivalent to estimating the number of iterations until the partitions converge. Which could be interesting to look at, though probably challenging. Furthermore, this algorithm requires ``infinite'' communication in the sense we let it run forever while the cluster heads converge. However, if it were somehow possible for agents to detect when the clusters had converged, this might be helpful in generating a termination condition for the algorithm.
\begin{center}
\begin{figure}
\includegraphics[width=1.0\textwidth]{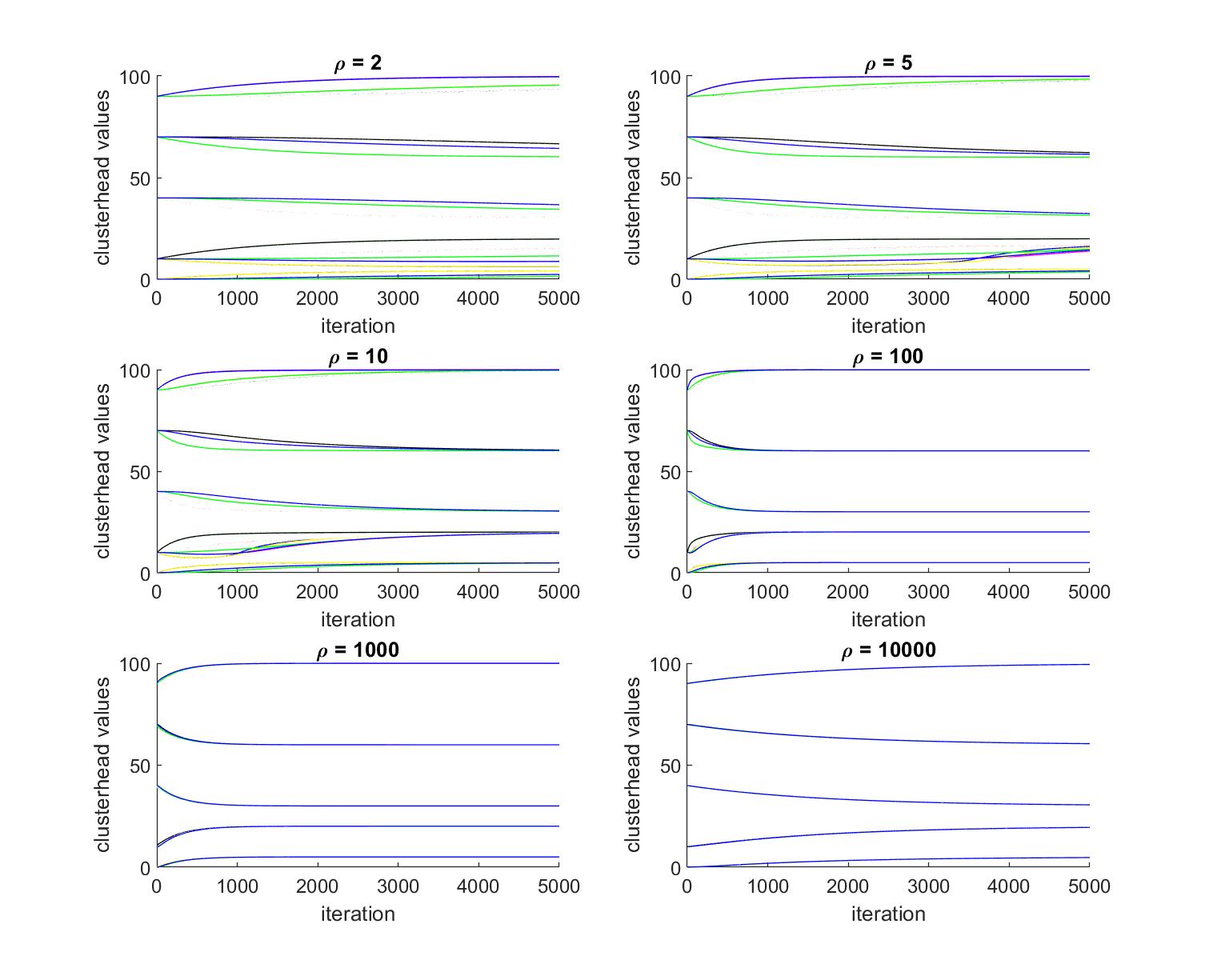}
\caption{}
\label{fig:plot1}
\end{figure}
\end{center}
\begin{table}[]
\label{table:tab1}
\centering
\begin{tabular}{|l|c|}
\hline
$\rho$ & \begin{tabular}[c]{@{}c@{}}Partition \\ Convergence Time\end{tabular} \\ \hline
2      & $>5000$                                                               \\ \hline
5      & 4026                                                                  \\ \hline
10     & 1225                                                                  \\ \hline
$10^2$ & 123                                                                   \\ \hline
$10^3$ & 164                                                                   \\ \hline
$10^4$ & 1015                                                                  \\ \hline
\end{tabular}
\end{table}

\vspace{-1em}
We note that if agents were not permitted to share data and were to perform clustering using their individual data alone, the cluster heads obtained at each individual agent would be an extremely poor clustering for the full dataset. This example is constructed to emphasize the effects of a disparity between data obtained at different nodes. Heterogeneity between datasets at different nodes can be common (and desirable) in practice given, for example, geographic separation between nodes.

\section{Convergence of $NK$-means} \label{sec:NKmeansconv}
We will now establish the convergence of the $NK$-means algorithm to the set of generalized minima of $\myQ^{\rho}(\cdot)$. In particular, we will prove Theorem \ref{th:conv}.

We will proceed as follows: We will begin by showing some preliminary results. We will then consider properties of limit points of the algorithm and then we will prove convergence of the algorithm to the set $\calM^\rho$. Theorem \ref{th:conv} will then follow immediately from Lemmas \ref{lm:limpoint} and \ref{lm:ind}.

We begin by proving the following monotonic cost improvement property for the $NK$-means algorithm. We use $\mu_t$ to denote the vector $\vecc_{m,k}\mu_m^k(t)$ at iteration $t$.
\begin{lemma}
\label{lm:costim} Let $\{\mathbf{x}_{t},\mathcal{C}_{t}\}$ be the sequence of cluster centers and clusters generated by the iterative procedure~\eqref{reassign1}--\eqref{cent_up1} and let the weight parameter $\alpha$ satisfy Assumption~\ref{ass-alpha}. Then, there exists a constant $c(\alpha)>0$ such that for each $t\in \N_{+}$ we have
\begin{equation}
\label{lm:costim1}
J^{\rho}(\mathbf{x}_{t+1},\mathcal{C}_{t+1})\leq J^{\rho}(\mathbf{x}_{t},\mathcal{C}_{t}) - c(\alpha)\left\|\mathbf{x}_{t}-\bmu_{t+1}\right\|^{2}.
\end{equation}
\end{lemma}

\begin{proof}
For each $t$, denote by $W^{\rho}_{t}(\cdot)$ the non-negative function on $\mathbb{R}^{MKp}$ such that
\begin{equation}
\label{lm:costim2}
W^{\rho}_{t}(\mathbf{x})=J^{\rho}(\mathbf{x},\mathcal{C}_{t})
\end{equation}
for all $\mathbf{x}\in\mathbb{R}^{MKp}$, i.e., $W^{\rho}_{t}(\cdot)$ is the clustering cost as a function of the cluster centers $\mathbf{x}$ given the cluster assignments are fixed by $\mathcal{C}_{t}$. Further, let $N^{\ast} = \max(N_1,\ldots,N_M)$. \\

\noindent{\bf Claim 1}. $W^{\rho}_{t+1}(\mathbf{x}_{t})\leq J^{\rho}(\mathbf{x}_{t},\mathcal{C}_{t})$,~~~$\forall t$.\\
Indeed, note that the reassignment step, which allocates the data points at each agent to the respective closest cluster centers~\eqref{reassign1}--\eqref{cent_up1}, ensures that
\begin{equation}
\label{lm:costim3}
\sum_{m=1}^{M}\sum_{k=1}^{K}\sum_{\y\in\mathcal{C}_{m}^{k}(t+1)}\|\y-\mathbf{x}_{m}^{k}(t)\|^{2}\leq \sum_{m=1}^{M}\sum_{k=1}^{K}\sum_{\y\in\mathcal{C}_{m}^{k}(t)}\|\y-\mathbf{x}_{m}^{k}(t)\|^{2}.
\end{equation}
Hence,
\begin{align}
\label{lm:costim4}
W^{\rho}_{t+1}(\mathbf{x}_{t}) & = J^{\rho}(\mathbf{x}_{t},\mathcal{C}_{t+1})\nonumber \\ & =\frac{1}{\rho}\sum_{m=1}^{M}\sum_{k=1}^{K}\sum_{\y\in\mathcal{C}_{m}^{k}(t+1)}\|\y-\mathbf{x}_{m}^{k}(t)\|^{2}+\mathbf{x}_{t}^{\top}(L\otimes I_{Kp})\mathbf{x}_{t}\nonumber \\
& \leq \frac{1}{\rho}\sum_{m=1}^{M}\sum_{k=1}^{K}\sum_{\y\in\mathcal{C}_{m}^{k}(t)}\|\y-\mathbf{x}_{m}^{k}(t)\|^{2}+\mathbf{x}_{t}^{\top}(L\otimes I_{Kp})\mathbf{x}_{t} \nonumber \\
& = J^{\rho}(\mathbf{x}_{t},\mathcal{C}_{t}).
\end{align}\\

\noindent{\bf Claim 2}. $W_{t+1}^{\rho}(\mathbf{x}_{t+1})\leq W^{\rho}_{t+1}(\mathbf{x}_{t})-2\alpha\left(d_{\min}-\alpha((1/\rho)N^{\ast}+\lambda_{M}(L))\right)\left\|\mathbf{x}_{t}-\bmu_{t+1}\right\|^{2}$,~~~$\forall t$.\\
For all $m$, $k$ and $t$, denote by $\beta_{m}^{k}(t)$ and $\gamma_{m}^{k}(t)$ the quantities
\begin{equation}
\label{lm:costim5}
\beta_{m}^{k}(t)=|\Omega_{m}|+\frac{1}{\rho}|\mathcal{C}_{m}^{k}(t)|~~\mbox{and}~~\gamma_{m}^{k}(t)=\frac{1}{\rho}|\mathcal{C}_{m}^{k}(t)|.
\end{equation}
Now, note that the function $W_{t+1}^{\rho}(\cdot)$ is convex and continuously differentiable in its argument $\mathbf{x}$ and hence we have
\begin{equation}
\label{lm:costim6}
W_{t+1}^{\rho}(\mathbf{x}_{t})\geq W_{t+1}^{\rho}(\mathbf{x}_{t+1})+\nabla W_{t+1}^{\rho}(\mathbf{x}_{t+1})^{\top}\left(\mathbf{x}_{t}-\mathbf{x}_{t+1}\right).
\end{equation}
Noting that
\begin{equation}
\label{lm:costim7}
\mathbf{x}_{t+1}=\mathbf{x}_{t}-\alpha\left(\mathbf{x}_{t}-\bmu_{t+1}\right),
\end{equation}
where the components of $\bmu_{t+1}$ are defined in~\eqref{cent_up1} and that the function $W^{\rho}_{t+1}(\cdot)$ may be written as
\begin{equation}
\label{lm:costim8}
W^{\rho}_{t+1}(\mathbf{x})=\frac{1}{\rho}\sum_{m=1}^{M}\sum_{k=1}^{K}\sum_{\y\in\mathcal{C}_{m}^{k}(t+1)}\|\y-\mathbf{x}_{m}^{k}\|^{2}+\mathbf{x}^{\top}(L\otimes I_{Kp})\mathbf{x},
\end{equation}
standard algebraic manipulations yield
\begin{equation}
\label{lm:costim9}
\nabla W_{t+1}^{\rho}(\mathbf{x}_{t+1})=2(\beta_{t+1}\otimes I_{p})\left(\mathbf{x}_{t}-\bmu_{t+1}\right)-\left(2\alpha(\gamma_{t+1}\otimes I_{p})+2\alpha(L\otimes I_{Kp})\right)\left(\mathbf{x}_{t}-\bmu_{t+1}\right),
\end{equation}
where $\beta_{t+1}$ and $\gamma_{t+1}$ are $MK\times MK$ diagonal matrices with $\beta_{t+1}=\diag(\beta_{m}^{k}(t+1))$ and $\gamma_{t+1}=\diag(\gamma_{m}^{k}(t+1))$.

Since for all $m,k,t$,
\begin{equation}
\label{lm:costim10}
\beta_{m}^{k}(t)=|\Omega_{m}|+\frac{1}{\rho}|\mathcal{C}_{m}^{k}(t)|\geq d_{\min}
\end{equation}
and
\begin{equation}
\label{lm:costim11}
\gamma_{m}^{k}(t)=\frac{1}{\rho}|\mathcal{C}_{m}^{k}(t)|\leq (1/\rho)N^{\ast},
\end{equation}
it may be readily verified that
\begin{equation}
\label{lm:costim12}
\mathbf{z}^{\top}(\beta_{t+1}\otimes I_{p})\mathbf{z}\geq d_{\min}\|\mathbf{z}\|^{2}
\end{equation}
and
\begin{equation}
\label{lm:costim13}
\mathbf{z}^{\top}\big(2\alpha(\gamma_{t+1}\otimes I_{p})+2\alpha(L\otimes I_{Kp})\big)\mathbf{z}\leq 2\alpha\big((1/\rho)N^{\ast}+\lambda_{M}(L)\big)\|\mathbf{z}\|^{2}
\end{equation}
for any $\mathbf{z}\in\mathbb{R}^{MKp}$. Hence, by~\eqref{lm:costim6} and~\eqref{lm:costim9} we obtain
\begin{align}
\label{lm:costim14}
W_{t+1}^{\rho}(\mathbf{x}_{t}) \geq & W_{t+1}^{\rho}(\mathbf{x}_{t+1})+2\alpha\left(\mathbf{x}_{t}-\bmu_{t+1}\right)^{\top}(\beta_{t+1}\otimes I_{p})\left(\mathbf{x}_{t}-\bmu_{t+1}\right)\nonumber \\
& - \alpha\left(\mathbf{x}_{t}-\bmu_{t+1}\right)^{\top}\left(2\alpha(\gamma_{t+1}\otimes I_{p})+2\alpha(L\otimes I_{Kp})\right)\left(\mathbf{x}_{t}-\bmu_{t+1}\right)\nonumber \\
\geq & W_{t+1}^{\rho}(\mathbf{x}_{t+1}) +2\alpha d_{\min}\left\|\mathbf{x}_{t}-\bmu_{t+1}\right\|^{2}\\ &-2\alpha^{2}((1/\rho)N^{\ast}+\lambda_{M}(L))\left\|\mathbf{x}_{t}-\bmu_{t+1}\right\|^{2}\nonumber \\
= & W^{\rho}_{t+1}(\mathbf{x}_{t+1})+2\alpha\left(d_{\min}-\alpha((1/\rho)N^{\ast}+\lambda_{M}(L))\right)\left\|\mathbf{x}_{t}-\bmu_{t+1}\right\|^{2}
\end{align}
and Claim 2 follows.

To complete the final steps of the proof of Lemma~\ref{lm:costim}, we note that the quantity
\begin{equation}
\label{lm:costim15}
c(\alpha)=2\alpha\left(d_{\min}-\alpha((1/\rho)N^{\ast}+\lambda_{M}(L))\right)
\end{equation}
satisfies $c(\alpha)>0$ by Assumption~\ref{ass-alpha}. The assertion then follows immediately from Claims 1--2 and additionally noting that, by definition, $W^{\rho}_{t+1}(\mathbf{x}_{t+1})=J^{\rho}(\mathbf{x}_{t+1},\mathcal{C}_{t+1})$.
\end{proof}

As an immediate corollary we obtain the following.
\begin{corollary}
\label{corr:costim} Let $\{\mathbf{x}_{t}\}$ be the sequence of cluster centers generated by the iterative procedure~\eqref{reassign1}--\eqref{cent_up1} and let the weight parameter $\alpha$ satisfy Assumption~\ref{ass-alpha}. Then, for each $t$ we have
\begin{equation}
\label{corr:costim1}
\myQ^{\rho}(\mathbf{x}_{t+1})\leq \myQ^{\rho}(\mathbf{x}_{t}) - c(\alpha)\left\|\mathbf{x}_{t}-\bmu_{t+1}\right\|^{2},
\end{equation}
where $c(\alpha)>0$ is defined in~\eqref{lm:costim1}.
\end{corollary}
\begin{proof}
Note that, by definition of $\mathcal{U}_{\vx}$ (see Section \ref{sec:gen-min}), we have for each $\widehat{\mathcal{C}}\in\mathcal{U}_{\vx_t}$
\begin{align}
\label{corr:costim2}
\myQ^{\rho}(\mathbf{x}_{t}) = \sum_{m=1}^{M}\left(\frac{1}{\rho}\sum_{\mathbf{y}\in\mathcal{D}_{m}}\min_{1\leq k\leq K}\|\y-\mathbf{x}_{m}^{k}(t)\|^{2}+\sum_{l\in\Omega_{m}}\|\mathbf{x}_{m}(t)-\mathbf{x}_{l}(t)\|^{2}\right) \nonumber \\ = \sum_{m=1}^{M}\left(\left(\frac{1}{\rho}\sum_{k=1}^{K}\sum_{\y\in\widehat{\mathcal{C}}_{m}^{k}}\|\y-\mathbf{x}_{m}^{k}(t)\|^{2}\right)+\left(\sum_{l\in\Omega_{m}}\|\mathbf{x}_{m}(t)-\mathbf{x}_{l}(t)\|^{2}\right)\right) \nonumber \\ =J^{\rho}(\mathbf{x}_{t},\widehat{\mathcal{C}}).
\end{align}
In particular, $\myQ^{\rho}(\mathbf{x}_{t})=J^{\rho}(\mathbf{x}_{t},\mathcal{C}_{t+1})$ since $\mathcal{C}_{t+1}\in\mathcal{U}_{\vx_t}$. Hence, by Lemma~\ref{lm:costim} we have
\begin{align}
\label{corr:costim3}
\myQ^{\rho}(\mathbf{x}_{t+1}) & =\min_{\mathcal{C}}J^{\rho}(\mathbf{x}_{t+1},\mathcal{C})\leq J^{\rho}(\mathbf{x}_{t+1},\mathcal{C}_{t+1})\nonumber \\
& \leq J^{\rho}(\mathbf{x}_{t},\mathcal{C}_{t+1}) - c(\alpha)\left\|\mathbf{x}_{t}-\bmu_{t+1}\right\|^{2}\nonumber \\
& = \myQ^{\rho}(\mathbf{x}_{t}) - c(\alpha)\left\|\mathbf{x}_{t}-\bmu_{t+1}\right\|^{2}.
\end{align}
\end{proof}

The following boundedness of the iterate sequence $\{\mathbf{x}_{t}\}$ is also immediate.
\begin{proposition}
\label{prop:bound} Let the hypotheses of Lemma~\ref{lm:costim} hold.
%and let $\mathcal{D}^{\ast}=\mathcal{D}_{1}\cup\cdots\cup\mathcal{D}_{M}$ denote the set of all agent data.
Let $\cco(\mathbf{x}_{0},\mathcal{D})$ denote the closed convex hull of the set of data points in $\mathcal{D}$ and the center initializations $\mathbf{x}_{m}^{k}(0)$, $m=1,\cdots,M$ and $k=1,\cdots,K$. Then, it holds that
\begin{equation}
\label{prop:bound1}
\mathbf{x}_{m}^{k}(t)\in\cco(\mathbf{x}_{0},\mathcal{D}),~~~\forall m,k,t.
\end{equation}
\end{proposition}
\begin{proof}
The proof follows by induction. Clearly, $\mathbf{x}_{m}^{k}(0)\in\cco(\mathbf{x}_{0},\mathcal{D})$ for all $m,k$. Suppose that the assertion holds for all times $s$ such that $0\leq s\leq t$. Now, observe that, by~\eqref{cent_up1}, for all $m,k$, $\bmu_{m}^{k}(t+1)$ is a convex combination of data points in $\mathcal{D}$ and the current cluster center estimates $\mathbf{x}_{l}^{k}(t)$, $m=1,\cdots,M$ and $k=1,\cdots,K$. Hence, by the induction hypothesis, i.e., $\mathbf{x}_{l}^{k}(t)\in\cco(\mathbf{x}_{0},\mathcal{D})$ for all $l,k$, it readily follows that
\begin{equation}
\label{prop:bound2}
\bmu_{m}^{k}(t+1)\in\cco(\mathbf{x}_{0},\mathcal{D}),~~~\forall m,k.
\end{equation}
Also, note that by~\eqref{cent_up},
\begin{equation}
\label{prop:bound3}
\mathbf{x}_{m}^{k}(t+1)=(1-\alpha)\mathbf{x}_{m}^{k}(t)+\alpha\bmu_{m}^{k}(t+1),~~~\forall m,k.
\end{equation}
Further, by Assumption~\ref{ass-alpha}, $\alpha\leq 1$, since $d_{\min}\leq\lambda_{M}(L)$ by standard properties of the Laplacian. Hence, by~\eqref{prop:bound2}--\eqref{prop:bound3}, $\mathbf{x}_{m}^{k}(t+1)\in\cco(\mathbf{x}_{0},\mathcal{D})$ for all $m,k$, being a convex combination of points in $\cco(\mathbf{x}_{0},\mathcal{D})$. This completes the induction step and the assertion follows.
\end{proof}

\vspace{1em}
\noindent{\bf Analysis of Limit Points}. We will now show that limit points of the distributed algorithm~\eqref{reassign1}--\eqref{cent_up1} are generalized minima in the sense of Definition~\ref{def:locmin}.

We start with the following intermediate result which shows that the set-valued mapping $\calU_\vx$ is continuous in an appropriate sense (namely, $\vx\mapsto \calU_\vx$ is \emph{upper hemicontinuous} \citep{aubin2009set}).
\begin{lemma}
\label{lm:int_con}
Let Assumptions \ref{ass:data_size}, \ref{ass:comm-graph}, \ref{ass-connected}, and \ref{ass-alpha} hold.
For each $\bx\in\mathbb{R}^{MKp}$, a set of potential cluster centers, there exists $\Vap_{\bx}>0$ such that $\mathcal{U}_{\mathbf{x}}\subset \mathcal{U}_{\bx}$ for all $\mathbf{x}\in\mathbb{R}^{MKp}$ with $\|\mathbf{x}-\bx\|<\Vap_{\bx}$.
\end{lemma}
\begin{proof}
Let $\bx\in \R^{KMp}$ be fixed.
For each $m$ and $\vy\in\calD_m$ let $\omega_{m,\vy}$ be the subset of indices in $\{1,\cdots,K\}$ such that $k\in\omega_{m,\vy}$ if and only if
\begin{equation}
\label{lm:int_con1}
\|\y-\bx_{m}^{k}\|\leq\|\y-\bx_{m}^{\acute{k}}\|~~~\forall \acute{k}\in\{1,\cdots,K\}.
\end{equation}
Note that, by the above construction, for each $m$ and $\vy\in \calD_m$, the quantity
\begin{equation}
\label{lm:int_con2}
\min_{\acute{k}\notin\omega_{m,\vy}}\|\y-\bx_{m}^{\acute{k}}\|-\min_{\acute{k}}\|\y-\bx_{m}^{\acute{k}}\|
\end{equation}
is strictly positive (could be $\infty$), where we adopt the convention that the minimum of an empty set is $\infty$. Since the total number of data points across all the agents is finite, there exists $\Vap>0$ such that
\begin{equation}
\label{lm:int_con3}
\min_{m,\vy\in \calD_m}\left(\min_{\acute{k}\notin\omega_{m,\vy}}\|\y-\bx_{m}^{\acute{k}}\|-\min_{\acute{k}}\|\y-\bx_{m}^{\acute{k}}\|\right)>\Vap.
\end{equation}
Now, consider any $\mathbf{x}\in\mathbb{R}^{MKp}$ such that
\begin{equation}
\label{lm:int_con4}
\|\mathbf{x}_{m}^{k}-\bx_{m}^{k}\|<\Vap/2~~~\forall m,k.
\end{equation}
We now show that $\mathcal{U}_{\mathbf{x}}\subset\mathcal{U}_{\bx}$ for all $\mathbf{x}$ satisfying~\eqref{lm:int_con4}.

To this end, let $\mathcal{C}\in\mathcal{U}_{\mathbf{x}}$ and assume on the contrary that $\mathcal{C}\notin\mathcal{U}_{\bx}$. Then, by the construction above, there exist a triple $(m,n,k)$ such that $\y\in\mathcal{C}_{m}^{k}$ and $k\notin\omega_{m,\vy}$. Also, by definition, since $\y\in\mathcal{C}_{m}^{k}$ and $\mathcal{C}_{m}^{k}\in\mathcal{U}_{\mathbf{x}}$, we have that
\begin{equation}
\label{lm:int_con5}
\|\y-\mathbf{x}_{m}^{k}\|\leq\|\y-\mathbf{x}_{m}^{\acute{k}}\|~~~\forall \acute{k}\in\{1,\cdots,K\}.
\end{equation}
Hence, by~\eqref{lm:int_con4}--\eqref{lm:int_con5}, we have for all $\acute{k}$
\begin{align}
\label{lm:int_con6}
\|\y-\bx_{m}^{k}\|< \|\y-\mathbf{x}_{m}^{k}\|+\Vap/3\leq \|\y-\mathbf{x}_{m}^{\acute{k}}\|+\Vap/2\nonumber
%\\ <\|\y-\bx_{m}^{\acute{k}}\|+2\Vap/3.
\end{align}
In particular, letting $k_{0}\in\omega_{m,\vy}$ we have
\begin{equation}
\label{lm:int_con7}
\|\y-\bx_{m}^{k}\| < \|\y-\bx_{m}^{k_{0}}\|+\Vap/2.
\end{equation}
On the other hand, by~\eqref{lm:int_con3}, since $k\notin\omega_{m,\vy}$ we have that
\begin{equation}
\label{lm:int_con8}
\|\y-\bx_{m}^{k}\|>\|\y-\bx_{m}^{k_{0}}\|+\Vap,
\end{equation}
which clearly contradicts~\eqref{lm:int_con7}. Hence, we conclude that $\mathcal{C}\in\mathcal{U}_{\bx}$ and, more importantly, that $\mathcal{U}_{\mathbf{x}}\subset\mathcal{U}_{\bx}$ for all $\mathbf{x}\in\mathbb{R}^{MKp}$ satisfying~\eqref{lm:int_con4}. Hence, the desired assertion follows by taking $\Vap_{\bx}$ to be $\Vap/2$.
\end{proof}

The following lemma considers properties of limit points of the $NK$-means algorithm. It does not, however, establish that a limit exists.
\begin{lemma}
\label{lm:limpoint} Let Assumptions \ref{ass:data_size}, \ref{ass:comm-graph}, \ref{ass-connected}, and \ref{ass-alpha} hold. Then, any limit point $\bx$ of the sequence $\{\mathbf{x}_{t}\}$ of cluster centers is a generalized minimum in the sense of Definition~\ref{def:locmin}.
\end{lemma}
\begin{proof}
%By Remark~\ref{rem:deflocmin}
By the definition of $\overline{\calU}_{\vx}$, it suffices to show that the set $\overline{\mathcal{U}}_{\bx}$ is non-empty. To this end, let $\{\mathbf{x}_{t_{s}}\}_{s\geq 0}$ be a subsequence of $\{\mathbf{x}_{t}\}$, the sequence of cluster centers generated by the distributed algorithm, such that $\mathbf{x}_{t_{s}}\rightarrow\bx$ as $t\rightarrow\infty$. Note that by Proposition \ref{prop:bound} such a subsequence exists.

First, note that by Lemma \ref{lm:int_con} the following claim holds.\\

\noindent\textbf{Claim 1}. There exists $s_{0}$ sufficiently large such that $\mathcal{C}_{t_{s+1}}\in\mathcal{U}_{\bx}$ for all $s\geq s_{0}$.\\

Now, for the sake of contradiction suppose that $\overline{\mathcal{U}}_{\bx}=\emptyset$. This implies that for each $\widehat{\mathcal{C}}\in\mathcal{U}_{\bx}$ there exist $m$ and $k$ such that
\begin{equation}
\label{lm:limp9}
\bx_{m}^{k}\neq\frac{(1/\rho)\sum_{\y\in\widehat{\mathcal{C}}_{m}^{k}}\y+\sum_{l\in\Omega_{m}}\bx_{l}^{k}}{(1/\rho)|\widehat{\mathcal{C}}_{m}^{k}|+|\Omega_{m}|}.
\end{equation}
Since $\mathcal{U}_{\bx}$ consists of a finite number of elements, we have that
\begin{equation}
\label{lm:limp10}
\Vap_{1}=\min_{\widehat{\mathcal{C}}\in\mathcal{U}_{\bx}}\left\|\bx-\bmu(\bx,\widehat{\mathcal{C}})\right\|>0,
\end{equation}
where $\bmu(\bx,\widehat{\mathcal{C}})=\vecc_{m,k}\left(\bmu_{m}^{k}(\bx,\widehat{\mathcal{C}})\right)$ with
\begin{equation}
\label{lm:limp11}
\bmu_{m}^{k}(\bx,\widehat{\mathcal{C}})=\frac{(1/\rho)\sum_{\y\in\widehat{\mathcal{C}}_{m}^{k}}\y+\sum_{l\in\Omega_{m}}\bx_{l}^{k}}{(1/\rho)|\widehat{\mathcal{C}}_{m}^{k}|+|\Omega_{m}|}
\end{equation}
for each $m,k$. Recall that by Claim 1 there exists $s_{0}$ such that $\mathcal{C}_{t_{s+1}}\in\mathcal{U}_{\bx}$ for all $s\geq s_{0}$. Let $s_{1}\geq s_{0}$ be sufficiently large such that
\begin{equation}
\label{lm:limp12}
\left(1+\frac{\sqrt{MK}d_{\max}}{d_{\min}}\right)\left\|\mathbf{x}_{t_{s}}-\bx\right\|<\frac{\Vap_{1}}{2}
\end{equation}
for all $s\geq s_{1}$.

By the triangle inequality, we obtain for each $s$
\begin{equation}
\label{lm:limp13}
\|\mathbf{x}_{t_{s}}-\bmu_{t_{s+1}}\|\geq \|\bx-\bmu(\bx,\mathcal{C}_{t_{s+1}})\|-\|\mathbf{x}_{t_{s}}-\bx\|-\|\bmu(\bx,\mathcal{C}_{t_{s+1}})-\bmu_{t_{s+1}}\|.
\end{equation}
Note that by \eqref{corr:costim2} we have
\begin{align}
\label{lm:limp14}
\left\|\bmu_{m}^{k}(\bx,\mathcal{C}_{t_{s+1}})-\bmu_{m}^{k}(t_{s}+1)\right\|&=\left\|\frac{\sum_{l\in\Omega_{m}}(\bx_{l}^{k}-\mathbf{x}_{l}^{k}(t_{s}))}{(1/\rho)|\widehat{\mathcal{C}}_{m}^{k}(t_{s+1})|+|\Omega_{m}|}\right\|\nonumber \\ & \leq \frac{d_{\max}}{d_{\min}}\|\mathbf{x}_{t_{s}}-\bx\|
\end{align}
for all $s$. Hence, by~\eqref{lm:limp12}, we have for all $s\geq s_{1}$
\begin{equation}
\label{lm:limp15}
\left\|\mathbf{x}_{t_{s}}-\bx\|+\|\bmu(\bx,\mathcal{C}_{t_{s+1}})-\bmu_{t_{s+1}}\right\|\leq \left(1+\frac{\sqrt{MK}d_{\max}}{d_{\min}}\right)\left\|\mathbf{x}_{t_{s}}-\bx\right\|<\frac{\Vap_{1}}{2}.
\end{equation}
Also, note that $\mathcal{C}_{t_{s+1}}\in\mathcal{U}_{\bx}$ for all $s\geq s_{1}$ and hence, by~\eqref{lm:limp10} we have
\begin{equation}
\label{lm:limp16}
\|\bx-\bmu(\bx,\mathcal{C}_{t_{s+1}})\|\geq \Vap_{1}~~~\forall s\geq s_{1}.
\end{equation}
Substituting~\eqref{lm:limp15}--\eqref{lm:limp16} in~\eqref{lm:limp13} we obtain
\begin{equation}
\label{lm:limp17}
\|\mathbf{x}_{t_{s}}-\bmu_{t_{s+1}}\|>\Vap_{1}/2~~~\forall s\geq s_{1}.
\end{equation}

Note that, by Lemma~\ref{lm:costim}, we have for each $t$
\begin{equation}
\label{lm:limp18}
J^{\rho}(\mathbf{x}_{t+1},\mathcal{C}_{t+1})\leq J^{\rho}(\mathbf{x}_{t},\mathcal{C}_{t}) - c(\alpha)\left\|\mathbf{x}_{t}-\bmu_{t+1}\right\|^{2},
\end{equation}
and hence unrolling the recursion
\begin{equation}
\label{lm:limp19}
J^{\rho}(\mathbf{x}_{t},\mathcal{C}_{t})\leq J^{\rho}(\mathbf{x}_{0},\mathcal{C}_{0})-c(\alpha)\sum_{r=0}^{t-1}\left\|\mathbf{x}_{r}-\bmu_{r+1}\right\|^{2}.
\end{equation}
Since, by~\eqref{lm:limp17},
\begin{equation}
\label{lm:limp20}
\sum_{r=0}^{\infty}\left\|\mathbf{x}_{r}-\bmu_{r+1}\right\|^{2}\geq\sum_{s\geq s_{1}}^{\infty}\left\|\mathbf{x}_{t_{s}}-\bmu_{t_{s+1}}\right\|^{2}=\infty,
\end{equation}
we obtain from~\eqref{lm:limp19} that
\begin{equation}
\label{lm:limp21}
\limsup_{t\rightarrow\infty}J^{\rho}(\mathbf{x}_{t},\mathcal{C}_{t})=-\infty.
\end{equation}
Clearly,~\eqref{lm:limp21} contradicts the fact that the clustering cost $J(\cdot,\cdot)$ is non-negative and we conclude that the assertion $\overline{\mathcal{U}}_{\bx}=\emptyset$ is false. Hence, $\overline{\mathcal{U}}_{\bx}\neq\emptyset$ and we conclude that the limit point $\bx$ is a generalized minimum in the sense of Definition~\ref{def:locmin}.
\end{proof}

A few more manipulations yield a stronger result concerning the limiting behavior of clusters.
\begin{lemma}
\label{lm:limpc} Let Assumptions \ref{ass:data_size}, \ref{ass:comm-graph}, \ref{ass-connected} and \ref{ass-alpha} hold, let $\bx$ be a limit point of the sequence $\{\mathbf{x}_{t}\}$ of cluster centers generated by the distributed algorithm~\eqref{reassign1}--\eqref{cent_up1}, and let $\{\mathbf{x}_{t_{s}}\}_{s\geq 0}$ be a subsequence such that $\mathbf{x}_{t_{s}}\rightarrow\bx$ as $s\rightarrow\infty$. Then, there exists $\acute{s}$ large enough such that $\mathcal{C}_{t_{s+1}}\in\overline{\mathcal{U}}_{\bx}$ for all $s\geq\acute{s}$.
\end{lemma}
\begin{proof} Recall that $\overline{\mathcal{U}}_{\bx}$ is non-empty by Lemma \ref{lm:limpoint}. Now suppose on the contrary that the assertion in Lemma~\ref{lm:limpc} does not hold. Then, there exists a further subsequence $\{\mathbf{x}_{\acute{t}_{r}}\}_{r\geq 0}$, which is a subsequence of $\{\mathbf{x}_{t_{s}}\}_{s\geq 0}$, such that $\mathcal{C}_{\acute{t}_{r}+1}\notin\overline{\mathcal{U}}_{\bx}$ for all $r\geq 0$. Note that the subsequence $\{\mathbf{x}_{\acute{t}_{r}}\}$ inherits the limiting properties of the subsequence $\{\mathbf{x}_{t_{s}}\}$, and hence, by Lemma \ref{lm:int_con} we conclude that there exists $r_{0}$ sufficiently large such that
\begin{equation}
\label{lm:limpc1}
\mathcal{C}_{\acute{t}_{r}+1}\in\mathcal{U}_{\bx}~~~\forall r\geq r_{0}.
\end{equation}
Since, by construction, $\mathcal{C}_{\acute{t}_{r}+1}\notin\overline{\mathcal{U}}_{\bx}$ for all $r\geq 0$,~\eqref{lm:limpc1} implies that the set $\mathcal{U}_{\bx}\setminus\overline{\mathcal{U}}_{\bx}$ is non-empty and hence,
\begin{equation}
\label{lm:limpc2}
\Vap_{2}=\min_{\widehat{\mathcal{C}}\in\mathcal{U}_{\bx}\setminus\overline{\mathcal{U}}_{\bx}}\left\|\bx-\bmu(\bx,\widehat{\mathcal{C}})\right\|>0,
\end{equation}
where $\bmu(\bx,\widehat{\mathcal{C}})$ is defined as in~\eqref{lm:limp11}. Note that $\mathbf{x}_{\acute{t}_{r}}\rightarrow\bx$ as $r\rightarrow\infty$ and let $r_{1}\geq r_{0}$ be sufficiently large such that
\begin{equation}
\label{lm:limpc3}
\left(1+\frac{\sqrt{MK}d_{\max}}{d_{\min}}\right)\left\|\mathbf{x}_{\acute{t}_{r}}-\bx\right\|<\frac{\Vap_{2}}{2}
\end{equation}
for all $r\geq r_{1}$. Now, repeating the arguments as in~\eqref{lm:limp13}--\eqref{lm:limp17} (applied to the subsequence $\{\mathbf{x}_{\acute{t}_{r}}\}$) we obtain
\begin{equation}
\label{lm:limpc4}
\|\mathbf{x}_{\acute{t}_{r}}-\bmu_{\acute{t}_{r}+1}\|>\Vap_{2}/2~~~\forall r\geq r_{1}.
\end{equation}
By similar reasoning as in~\eqref{lm:limp18}--\eqref{lm:limp21} we finally derive the conclusion that
\begin{equation}
\label{lm:limpc5}
\limsup_{t\rightarrow\infty}J^{\rho}(\mathbf{x}_{t},\mathcal{C}_{t})=-\infty.
\end{equation}
Clearly,~\eqref{lm:limpc5} contradicts the fact that the clustering cost $J(\cdot,\cdot)$ is non-negative and hence the assertion that $\mathcal{C}_{t_{s+1}}\notin\overline{\mathcal{U}}_{\bx}$ infinitely often (i.o.) is false. Lemma~\ref{lm:limpc} now follows immediately.
\end{proof}

\begin{remark}[Finte-Time Convergence of Partitions] \label{remark:partition-convergence}
The above Lemma shows that if $\vx_t \to\bx$, then after some finite number of iterations $\acute{t}$,  the partitions generated by the $NK$-means algorithm satisfy $\calC_t \in \calU_{\bx}$ for all $t\geq \acute{t}$. We note that, if the limit point $\bx$ is such that each $\bx^k_m$, $k=1,\ldots,K$, $m=1,\ldots,M$ in $\bx$ is not equidistant from any two datapoints then $\calU_{\bx}$ will consist of a unique partition. In such cases, the above result implies that the sequence of partitions $\{\calC_t\}$ generated by the $NK$-means algorithm converges (to the unique partition) in finite time. Aside from exceptional cases, in practice we generally expect the sequence of partitions to converge in finite time.
\end{remark}

%\begin{remark}
%Suppose $\bx$ is a Lloyd's minimum and let $\hat \vx$ be the $M$-fold repetition of $\bx$. By Lemma ??, for all $\tilde \vx$ in a neighborhood of $\hat \vx$ there holds $\calU_{\tilde\vx} \subset \calU_{\hat\vx}$. If $\bx$ is such that each $\bx^k$, $k=1,\ldots,K$ is not equidistant from any two datapoints, then $\calU_{\hat\vx}$ will consist of a unique partition. In this case, if $\bx^\rho \to \hat \vx$, then any partition $\calP^\rho$ of $\calD$ generated by $\bx^\rho$ will converge to the unique partition $\calP$ in $\calU_{\hat x}$ in finite time. Moreover, $\calP$ is a Lloyd's minimum clustering in the sense that the pair $(\bx,\calP)\in \calL$.
%\end{remark}

$~$\newline
\noindent{\bf Convergence}. We now establish the convergence of the distributed scheme~\eqref{reassign1}--\eqref{cent_up1} to a generalized minimum.
The following result is crucial to establishing the convergence of the distributed clustering procedure.
\begin{lemma}[Local stability of generalized minima]
\label{lm:ind} Let Assumptions \ref{ass:data_size}, \ref{ass:comm-graph}, \ref{ass-connected}, and~\ref{ass-alpha} hold and let $\{\mathbf{x}_{t}\}$ be the sequence of cluster centers generated by the distributed procedure~\eqref{reassign1}--\eqref{cent_up1}. Suppose that $\bx$ is a generalized minimum in the sense of Definition~\ref{def:locmin}. Then there exists $\overline{\Vap}_{\bx}>0$, small enough, such that, for each $\Vap\in (0,\overline{\Vap}_{\bx})$ there exists $t_{\Vap}$ with the following property\textup{:} If
\begin{equation}
\label{lm:ind2}
\|\mathbf{x}_{\acute{t}}-\bx\|_{\infty}\leq\Vap
\end{equation}
for some $\acute{t}\geq t_{\Vap}$, then
\begin{equation}
\label{lm:ind3}
\|\mathbf{x}_{t}-\bx\|_{\infty}\leq\Vap,~~~\forall t\geq \acute{t}.
\end{equation}
\end{lemma}
\begin{proof} In what follows we will explicitly determine $\overline{\Vap}_{\bx}$ and $t_{\Vap}$ for all $\Vap\in (0,\overline{\Vap}_{\bx})$ and show that if $|\mathbf{x}_{\acute{t}}-\bx\|_{\infty}\leq\Vap$ for some $\acute{t}\geq t_{\Vap}$, then $\|\mathbf{x}_{\acute{t}+1}-\bx\|_{\infty}\leq\Vap$. The assertion~\eqref{lm:ind3} for all $t\geq \acute{t}$ will then follow by simple induction.

To this end, consider the following constructions: Recall Corollary~\ref{corr:costim} and, in particular,~\eqref{corr:costim1}, and let $\underline{\myQ}^{\rho}$ be the limit of the non-increasing sequence $\{\myQ^{\rho}(\mathbf{x}_{t})\}$ of clustering costs, i.e.,
\begin{equation}
\label{lm:ind4}
\underline{\myQ}^{\rho}=\lim_{t\rightarrow\infty}\myQ^{\rho}(\mathbf{x}_{t})=\inf_{t\geq 0}\myQ^{\rho}(\mathbf{x}_{t}).
\end{equation}
Since $\bx$ is a generalized minimum of $\myQ^\rho(\cdot)$, by \eqref{eq:calU-condition} we have $\overline{\calU}_{\bx} \not=\emptyset$. Note that for each $\widehat{\mathcal{C}}\in\mathcal{U}_{\bx}\setminus\overline{\mathcal{U}}_{\bx}$
\begin{equation}
\label{lm:ind5}
\|\bmu(\bx,\widehat{\mathcal{C}})-\bx\|>0,
\end{equation}
where $\bmu(\bx,\widehat{\mathcal{C}})$ is defined in~\eqref{lm:limp11}. Since the cardinality of $\mathcal{U}_{\bx}$ is finite, we conclude that there exists $\Vap_{2}>0$ (could be $\infty$) such that
\begin{equation}
\label{lm:ind6}
\min_{\widehat{\mathcal{C}}\in\mathcal{U}_{\bx}\setminus\overline{\mathcal{U}}_{\bx}}\|\bmu(\bx,\widehat{\mathcal{C}})-\bx\|>\Vap_{2}.
\end{equation}
Now, recall the constant $c(\alpha)$ in~\eqref{corr:costim1}, the positive constant $\Vap_{\bx}$ in Lemma~\ref{lm:int_con}, and define $\overline{\Vap}_{\bx}$ to be
\begin{equation}
\label{lm:ind7}
\overline{\Vap}_{\bx}=\min\left(\frac{\Vap_{\bx}}{\sqrt{MK}},c_{2}^{-1}\Vap_{2}\right),
\end{equation}
where $\Vap_{2}$ is defined in~\eqref{lm:ind6} and $c_{2}=1+\frac{\sqrt{MK}d_{\max}}{d_{\min}}$.

Finally, for each $\Vap\in (0,\overline{\Vap}_{\bx})$ choose $t_{\Vap}$ to be such that
\begin{equation}
\label{lm:ind8}
\myQ^{\rho}(\mathbf{x}_{t})\leq\underline{\myQ}^{\rho}+\frac{c(\alpha)}{2}\left(\Vap_{2}-c_{2}\Vap\right)^{2}~~~\forall t\geq t_{\Vap}.
\end{equation}
Note that such choice of $t_{\Vap}$ exists by~\eqref{lm:ind4} and the fact that $(\Vap_{2}-c_{2}\Vap)>0$ by~\eqref{lm:ind7}.

Now, fixing $\Vap\in (0,\overline{\Vap}_{\bx})$ we show that if $\acute{t}$ is such that $\acute{t}\geq t_{\Vap}$ and $\|\mathbf{x}_{\acute{t}}-\bx\|_{\infty}<\Vap$, then $\|\mathbf{x}_{\acute{t}+1}-\bx\|_{\infty}<\Vap$. This is accomplished in two steps: first (see Claim 1 below), we show that under the stated conditions $\mathcal{C}_{\acute{t}+1}\in\overline{\mathcal{U}}_{\bx}$ and subsequently $\|\mathbf{x}_{\acute{t}+1}-\bx\|_{\infty}<\Vap$.\\ \\

\noindent{\bf Claim 1}. Let $\Vap\in (0,\overline{\Vap}_{\bx})$, $\acute{t}\geq t_{\Vap}$, and $\|\mathbf{x}_{\acute{t}}-\bx\|_{\infty}<\Vap$. Then $\mathcal{C}_{\acute{t}+1}\in\overline{\mathcal{U}}_{\bx}$.\\

Suppose on the contrary that $\mathcal{C}_{\acute{t}+1}\notin\overline{\mathcal{U}}_{\bx}$.

From~\eqref{lm:ind7}
\begin{equation}
\label{lm:ind9}
\|\mathbf{x}_{\acute{t}}-\bx\|\leq \sqrt{MKp}\|\mathbf{x}_{\acute{t}}-\bx\|_{\infty}<\Vap_{\bx},
\end{equation}
and hence, by Lemma~\ref{lm:int_con}, $\mathcal{U}_{\mathbf{x}_{\acute{t}}}\subset\mathcal{U}_{\bx}$. This implies, see \eqref{reassign1}--\eqref{cent_up1}, that $\mathcal{C}_{\acute{t}+1}\in\mathcal{U}_{\bx}$.
Letting $\mu(\vx,\calC)$ be as defined in \eqref{lm:limp11} we then have, for each $m,k$,
\begin{align}
\label{lm:ind10}
 \left\|\bmu_{m}^{k}(\acute{t}+1)-\bmu_{m}^{k}(\bx,\mathcal{C}_{\acute{t}+1})\right\| = \nonumber \\
\left\|\frac{(1/\rho)\sum_{\y\in\mathcal{C}_{m}^{k}(\acute{t}+1)}\y+\sum_{l\in\Omega_{m}}\mathbf{x}_{l}^{k}(\acute{t})}{(1/\rho)|\mathcal{C}_{k}^{m}(\acute{t}+1)|+|\Omega_{m}|}-\frac{(1/\rho)\sum_{\y\in\mathcal{C}_{m}^{k}(\acute{t}+1)}\y+\sum_{l\in\Omega_{m}}\bx_{l}^{k}}{(1/\rho)|\mathcal{C}_{k}^{m}(\acute{t}+1)|+|\Omega_{m}|}\right\|\nonumber \\ = \frac{\left\|\sum_{l\in\Omega_{m}}\left(\mathbf{x}_{l}^{k}(\acute{t})-\bx_{l}^{k}\right)\right\|}{(1/\rho)|\mathcal{C}_{k}^{m}(\acute{t}+1)|+|\Omega_{m}|}\nonumber \\ \leq \frac{d_{\max}\|\mathbf{x}_{\acute{t}}-\bx\|_{\infty}}{d_{\min}}\\
\leq \frac{d_{\max}\Vap}{d_{\min}},
\end{align}
and thus
\begin{equation}
\label{lm:ind11}
\left\|\bmu_{\acute{t}+1}-\bmu(\bx,\mathcal{C}_{\acute{t}+1})\right\|\leq \frac{\sqrt{MK}d_{\max}\Vap}{d_{\min}}.
\end{equation}
Now, since $\mathcal{C}_{\acute{t}+1}\notin\overline{\mathcal{U}}_{\bx}$, we have (see~\eqref{lm:ind6})
\begin{equation}
\label{lm:ind12}
\left\|\bmu(\bx,\mathcal{C}_{\acute{t}+1})-\bx\right\|>\Vap_{2},
\end{equation}
which together with~\eqref{lm:ind11} leads to the estimate
\begin{align}
\label{lm:ind13}
\left\|\bmu_{\acute{t}+1}-\bx\right\|&\geq \left\|\bmu(\bx,\mathcal{C}_{\acute{t}+1})-\bx\right\|-\left\|\bmu_{\acute{t}+1}-\bmu(\bx,\mathcal{C}_{\acute{t}+1})\right\|\nonumber \\
&> \Vap_{2}-\frac{\sqrt{MK}d_{\max}\Vap}{d_{\min}}.
\end{align}
It then follows that
\begin{align}
\label{lm:ind14}
\left\|\mathbf{x}_{\acute{t}}-\bmu_{\acute{t}+1}\right\| & \geq \left\|\bmu_{\acute{t}+1}-\bx\right\|-\left\|\mathbf{x}_{\acute{t}}-\bx\right\|\nonumber \\
& >\Vap_{2}-\frac{\sqrt{MK}d_{\max}\Vap}{d_{\min}}-\Vap\\
& =\Vap_{2}-c_{2}\Vap,
\end{align}
where $c_{2}$ is as in~\eqref{lm:ind7}. Together with corollary~\ref{corr:costim},~\eqref{lm:ind8}, and the fact that $\acute{t}\geq t_{\Vap}$ this yields
\begin{align}
\label{lm:ind15}
\myQ^{\rho}(\mathbf{x}_{\acute{t}+1}) & \leq \myQ^{\rho}(\mathbf{x}_{\acute{t}})-c(\alpha)\|\mathbf{x}_{\acute{t}}-\bmu_{\acute{t}+1}\|^{2}\nonumber \\
& \leq \underline{\myQ}^{\rho}+(c(\alpha)/2)(\Vap_{2}-c_{2}\Vap)^{2}-c(\alpha)(\Vap_{2}-c_{2}\Vap)^{2}<\underline{\myQ}^{\rho}.
\end{align}
This clearly contradicts~\eqref{lm:ind4} and we conclude that $\mathcal{C}_{\acute{t}+1}\in\overline{\mathcal{U}}_{\bx}$, thus establishing Claim 1.\\ \\

To complete the proof, note that by~\eqref{cent_up} we have
\begin{equation}
\label{lm:ind16}
\left(\mathbf{x}_{\acute{t}+1}-\bx\right)=\left(1-\alpha\right)\left(\mathbf{x}_{\acute{t}}-\bx\right)+\alpha\left(\bmu_{\acute{t}+1}-\bx\right).
\end{equation}
Since $\mathcal{C}_{\acute{t}+1}\in\overline{\mathcal{U}}_{\bx}$, we have $\bmu(\bx,\mathcal{C}_{\acute{t}+1})=\bx$ and hence, by \eqref{lm:limp11} it follows that for all $m$ and $k$ we have
\begin{align}
\label{lm:ind17}
\left\|\bmu_{m}^{k}(\acute{t}+1)-\bx_m^k\right\|_{\infty} & \leq \frac{\sum_{l\in\Omega_{m}}\left\|\mathbf{x}_{l}^{k}(\acute{t})-\bx_{l}^{k}\right\|_{\infty}}{(1/\rho)|\mathcal{C}_{k}^{m}(\acute{t}+1)|+|\Omega_{m}|}\nonumber \\ & \leq \frac{|\Omega_{m}|.\|\mathbf{x}_{\acute{t}}-\bx\|_{\infty}}{|\Omega_{m}|}\\
& \leq \|\mathbf{x}_{\acute{t}}-\bx\|_{\infty}.
\end{align}
By~\eqref{lm:ind16}--\eqref{lm:ind17} we then obtain
\begin{align}
\label{lm:ind18}
\left\|\mathbf{x}_{\acute{t}+1}-\bx\right\|_{\infty}& \leq \left(1-\alpha\right)\|\mathbf{x}_{\acute{t}}-\bx\|_{\infty}+\alpha\|\mathbf{x}_{\acute{t}}-\bx\|_{\infty}\nonumber \\
& \leq \|\mathbf{x}_{\acute{t}}-\bx\|_{\infty}<\Vap.
\end{align}
In particular, we have shown that if $\Vap\in (0,\overline{\Vap}_{\bx})$ and $\acute{t}$ is such that $\acute{t}\geq t_{\Vap}$ and $\|\mathbf{x}_{\acute{t}}-\bx\|_{\infty}<\Vap$, then $\|\mathbf{x}_{\acute{t}+1}-\bx\|_{\infty}<\Vap$. A simple inductive argument yields that $\|\mathbf{x}_{t}-\bx\|_{\infty}<\Vap$ for all $t\geq\acute{t}$.\\
\end{proof}

{\bf Proof of Theorem \ref{th:conv}.}  Theorem \ref{th:conv} now follows immediately from Lemmas \ref{lm:limpoint}--\ref{lm:ind}.
In particular, the fact that $\{\vx_t\}$ converges to a generalized  minimum $\bx$ of $Q^\rho(\cdot)$ follows from Lemmas \ref{lm:limpoint} and \ref{lm:ind}. Having established the convergence of $\{\vx_t\}$ to $\bx$, the existence of a finite $T$ such that $\calC_t \in \calU_{\bx}$ for all  $t\geq T$ follows immediately from Lemma \ref{lm:limpc}.

\section{Generalized Minima and Lloyd's Minima}
\label{sec:genminimaconv}
In this section we will study asymptotic properties of the set of generalized minima of $Q^\rho$ as $\rho\to\infty$ and, in particular, we will prove Theorem \ref{lm:conv_to_Z_1}.

%\subsection{Limiting behavior of the set $\oM^{\rho}$ as $\rho\rightarrow\infty$}
%\label{subsec:limsets_local}
%Recall that the set $\oM^\rho$ is the subset of $\calM^\rho$ in which all clusters
To facilitate the discussion below, we introduce some notation. For $\mathbf{z}\in\mathbb{R}^{Kp}$, denote by $\mathcal{V}_{\mathbf{z}}$ the subset of partitions $\mathcal{P}$ of the collective data set $\mathcal{D}$ such that $\mathcal{P}=\left\{\mathcal{P}^{1},\mathcal{P}^{2},\cdots,\mathcal{P}^{K}\right\}\in\mathcal{V}_{\mathbf{z}}$ if, for all $\mathbf{y}\in\mathcal{D}$,
\begin{align}
\label{def_V_z}
\mathbf{y}\in\mathcal{P}^{k}~\Longrightarrow~\|\mathbf{y}-\mathbf{z}^{k}\|\leq\|\mathbf{y}-\mathbf{z}^{\acute{k}}\|,~\forall \acute{k}.
\end{align}
From the definition of Lloyd's minima, it follows that $\mathbf{z}\in\mathcal{Z}$ if and only if there exists a $\mathcal{P}\in\mathcal{V}_{\mathbf{z}}$ with the property that
\begin{align}
\label{def_V_z_1}
\left|\mathcal{P}^{k}\right|\mathbf{z}^{k}=\sum_{\y\in\mathcal{P}^{k}}\y
\end{align}
for all $k$.

We now prove Theorem \ref{lm:conv_to_Z_1}.
\begin{proof}[\textbf{Proof of Theorem \ref{lm:conv_to_Z_1}}] A necessary condition for $\vz\in\R^{Kp}$ to be an element of $\OZ$ is that there exist a partition $\calP = (\calP^k)_{k=1}^K$ such that for each  for each $k=1,\ldots,K$, the subvector $\vz^k$ is the centroid of $\calP^k$. Since the number of possible partitions is finite, this implies that $\OZ$ is finite, and in particular, compact.
Note that the sequence $\{\bx^{\rho}_{m}\}_{\rho\in\N_{+}}$, for each $m$, is bounded (by hypothesis). It then suffices to show that each limit point of $\{\bx^{\rho}_{m}\}_{\rho\in\N_{+}}$ belongs to $\OZ$.

To this end, without loss of generality, suppose that $\bx^{\rho}\rightarrow\mathbf{x}$ as $\rho\rightarrow\infty$. Then, necessarily, by Lemma~\ref{lm:rhocons} we have $\mathbf{x}=\mathbf{1}_{M}\otimes \mathbf{z}$ for some $\mathbf{z}\in\mathbb{R}^{Kp}$, i.e., the agent cluster centers reach consensus as $\rho\rightarrow\infty$. Clearly, $\mathbf{z}\in\cco(\mathcal{D})$. To claim the desired assertion, it is sufficient to show that $\mathbf{z}\in\mathcal{Z}$, which is achieved below.

Since $\bx^{\rho}\in\oM^{\rho}$, for each $\rho$, by Proposition~\ref{prop:rel_glob_loc} there exists a clustering $\calC^{\rho}\in\mathcal{U}_{\bx^{\rho}}$ such that the tuple $(\bx^{\rho},\calC^{\rho})\in\mathcal{J}^{\rho}$, and hence, in particular, we have for all $m$ and $k$
\begin{align}
\label{lm:conv_to_Z_3}
\left|\calC^{\rho,k}_{m}\right|\bx^{\rho,k}_{m}=\sum_{\y\in\calC^{\rho,k}_{m}}\y+\rho\sum_{l\in\Omega_{m}}\left(\bx^{\rho,k}_{l}- \bx^{\rho,k}_{m}\right).
\end{align}
By the symmetricity of the inter-agent communication graph we have
\begin{align}
\label{lm:conv_to_Z_4}
\sum_{m=1}^{M}\sum_{l\in\Omega_{m}}\left(\bx^{\rho,k}_{l}- \bx^{\rho,k}_{m}\right)=0,
\end{align}
and hence, summing both sides of~\eqref{lm:conv_to_Z_3} over $m$, we obtain for all $k$
\begin{align}
\label{lm:conv_to_Z_5}
\sum_{m}\left|\calC^{\rho,k}_{m}\right|\bx^{\rho,k}_{m}=\sum_{m}\sum_{\y\in\calC^{\rho,k}_{m}}\y.
\end{align}
For each $\rho$ and $k$, let
\begin{align}
\label{lm:conv_to_Z_6}\mathcal{P}^{\rho,k}=\calC_{1}^{\rho,k}\cup\calC_{2}^{\rho,k}\cup\cdots\cup\calC_{M}^{\rho,k},
\end{align}
and note that $\mathcal{P}^{\rho}=\left\{\mathcal{P}^{\rho,1},\mathcal{P}^{\rho,2},\cdots,\mathcal{P}^{\rho,K}\right\}$ is a valid partition of $\mathcal{D}$.

Since $\bx^{\rho}\rightarrow\mathbf{x}$ as $\rho\rightarrow\infty$, by Lemma~\ref{lm:int_con} there exists $\rho_{0}>0$ such that $\mathcal{U}_{\bx^{\rho}}\subset\mathcal{U}_{\mathbf{x}}$ for all $\rho\geq\rho_{0}$. Since $\mathbf{x}=\left(\mathbf{z},\mathbf{z},\cdots,\mathbf{z}\right)$ it then follows that (see~\eqref{def_V_z})
\begin{align}
\label{lm:conv_to_Z_7}\mathcal{P}^{\rho}\in\mathcal{V}_{\mathbf{z}},~~\forall \rho\geq\rho_{0}.
\end{align}
By~\eqref{lm:conv_to_Z_5}--\eqref{lm:conv_to_Z_6} and simple algebraic manipulations we obtain, for all $k$,
\begin{align}
\label{lm:conv_to_Z_8}
\left\|\left|\mathcal{P}^{\rho,k}\right|\mathbf{z}^{k}-\sum_{\y
\in\mathcal{P}^{\rho,k}}\y\right\|\leq\sum_{m}\left|\calC^{\rho,k}_{m}\right|\left\|\bx^{\rho,k}_m-\mathbf{z}^{k}\right\|.
\end{align}
Now fix $\Vap>0$ and choose $\rho(\Vap)\geq\rho_{0}$ such that
\begin{align}
\label{lm:conv_to_Z_9}
\left\|\bx^{\rho,k}-\mathbf{z}^{k}\right\|\leq\frac{\Vap}{M|\mathcal{D}|},~~\forall m,k.
\end{align}
Then $\mathcal{P}^{\rho(\Vap)}\in\mathcal{V}_{\mathbf{z}}$ and by~\eqref{lm:conv_to_Z_8} we obtain
\begin{align}
\label{lm:conv_to_Z_10}
\max_{k}\left\|\left|\mathcal{P}^{\rho(\Vap),k}\right|\mathbf{z}^{k}-\sum_{\y\in\mathcal{P}^{\rho(\Vap),k}}\y\right\| & \leq  \max_{k}\sum_{m}\left|\calC^{\rho(\Vap),k}_{m}\right|\left\|\bx^{\rho(\Vap),k}-\mathbf{z}^{k}\right\|\\ & \leq \frac{\Vap}{M|\mathcal{D}|}\left(\max_{k}\sum_{m}\left|\calC^{\rho(\Vap),k}_{m}\right|\right)\\ & \leq \Vap.
\end{align}
In other words, for each $\Vap>0$, there exists a valid partition $\mathcal{P}\doteq\mathcal{P}^{\rho(\Vap)}$ of $\mathcal{D}$ such that $\mathcal{P}^{\rho(\Vap)}\in\mathcal{V}_{\mathbf{z}}$ and
\begin{align}
\label{lm:conv_to_Z_11}
\max_{k}\left\|\left|\mathcal{P}^{\rho(\Vap),k}\right|\mathbf{z}^{k}-\sum_{\y\in\mathcal{P}^{\rho(\Vap),k}}\y\right\|\leq\Vap.
\end{align}
Since $\Vap>0$ is arbitrary, we have,
\begin{align}
\label{lm:conv_to_Z_12}
\min_{\mathcal{P}\in\mathcal{V}_{\mathbf{z}}}\max_{k}\left\|\left|\mathcal{P}^{\rho(\Vap),k}\right|\mathbf{z}^{k}-\sum_{\y\in\mathcal{P}^{\rho(\Vap),k}}\y\right\|=0.
\end{align}
Since the number of partitions in $\mathcal{V}_{\mathbf{z}}$ is finite, there exists $\mathcal{P}^{\ast}\in\mathcal{V}_{\mathbf{z}}$ such that
\begin{align}
\label{lm:conv_to_Z_13}
\max_{k}\left\|\left|\mathcal{P}^{\ast,k}\right|\mathbf{z}^{k}-\sum_{\y\in\mathcal{P}^{\ast,k}}\y\right\|=0.
\end{align}
Thus, by the equivalent characterization of Lloyd's minima in~\eqref{def_V_z}--\eqref{def_V_z_1}, we conclude that $\mathbf{z}\in\mathcal{Z}$ leading to the desired assertion.
\end{proof}

We remark that using similar arguments to the above proof it is straightforward to strengthen Theorem~\ref{lm:conv_to_Z} slightly achieving the following uniform convergence property.
\begin{corollary}
\label{corr:conv_to_Z}
Let Assumptions \ref{ass:data_size}, \ref{ass:comm-graph}, and \ref{ass-connected} hold. Then, we have that
\begin{equation}
\label{corr:conv_to_Z1}
\lim_{\rho\rightarrow\infty}\sup_{\bx^{\rho}\in\oM^{\rho}}\max_{m=1,\cdots,M}d\left(\bx_{m}^{\rho},\OZ\right)=0,
\end{equation}
where for each $\rho$, $\bx^{\rho}\in\oM^{\rho}$ and $m$, the quantity $\bx_{m}^{\rho}$ denotes the $K$-tuple $\{\bx_{m}^{\rho,1},\cdots,\bx_{m}^{\rho,K}\}$ of cluster center estimates at an agent $m$.

In particular, we have that, for each $\Vap>0$, there exists $\rho_{\Vap}\doteq\rho_{\Vap}(\mathcal{D},\mathcal{G})$, a function of the data set $\mathcal{D}$ and the inter-agent communication topology $\mathcal{G}$ only, such that
\begin{equation}
\label{corr:conv_to_Z2}
d\left(\bx_{m}^{\rho},\OZ\right)\leq\Vap
\end{equation}
for all $m$, $\rho\geq\rho_{\Vap}$ and $\bx^{\rho}\in\oM^{\rho}$.
\end{corollary}

\subsection{Convergence of Clusters for Finite $\rho$} \label{sec:finite-rho-conv}
We now prove Corollary \ref{cor:finite-rho}.
\begin{proof}
Suppose $(\vx,\calP)\in\calL$ is a Lloyd's minimum and let $\hat \vx = (\bx,\ldots,\bx)$ be the $M$-fold repetition of $\vx$. Suppose that $\breve \calC \in \calU_{\hat \vx}$ and let $\breve \calP$ be the partition of $\calD$ generated by $\breve\calC$ in the usual way \eqref{eq:generated-partition}. By the definition of $\calU_{\hat \vx}$ we see that
$\calH(\vx,\calP) = \calH(\vx,\breve\calP)$.

Since the set of Lloyd's minima is finite, By Lemma \ref{lm:int_con}, there exists an $\epsilon >0$ such that for each $\vx\in\calZ$, and $\hat \vx = (\vx,\ldots,\vx)$ (again, the $M$-fold repetition of $\vx$) and all $\tilde \vx$ within a ball of radius $\epsilon$ of $\vx$, there holds $\calU_{\tilde \vx} \subset \calU_{\hat \vx}$.

The result now follows from Theorem \ref{lm:conv_to_Z}.
\end{proof}

%\noindent \textcolor{red}{Items for discussion: As we discussed last time, we have to point out that the above implies that as $\rho$ increases agents converge to cluster centers that are very close to each other and very close to Lloyd's minima; the closeness to Lloyd's can be determined independently of which generalized minima our algorithm is converging to (we can emphasize why this property is helpful, among other things, robustness to initializations etc.; I also think that it may be worthwhile to give the main message of the paper as a single theorem that combines algorithm convergence and with appropriate $\rho$ selection achieving arbitrary closeness to Lloyd's minima and the consensus subspace. If we do this, where would be the right place to state such a result?}

%\subsection{Convergence to set of Lloyd's Minima for Finite $\rho$}
%[ADD RESULT HERE.]

\section{Global Minima} \label{sec:global-min-convergence}
We will now study properties of the set of global minima of $\myQ^\rho$ as $\rho\to\infty$ and, in particular,  we will prove Theorem \ref{lm:conv_to_Z}. We will begin by considering basic properties of set of global minima of $Q^\rho$ in Section \ref{subsec:prop_glob_minima}. In Section \ref{subsec:limsets} we will then consider the behavior of this set as $\rho\to\infty$ and we will give the proof of Theorem \ref{lm:gloptconv1}

\subsection{Properties of $\mathcal{M}^\rho_g$}
\label{subsec:prop_glob_minima}

We start with the following result which shows that, at a global minimum $(\bx,\breve{\calC})$ of  $J^\rho$, the partition $\calP = \{\calP^1,\ldots,\calP^K\}$ of $\calD$ formed as $\calP^k = \calC_1^k\cup\cdots\cup\calC_M^k$, $k=1,\ldots,K$, is non-degenerate in the sense that $\calP^k\not= \emptyset$ for any $k$.
\begin{lemma}
\label{lm:non_empty_clusters}
Let Assumptions \ref{ass:data_size}, \ref{ass:comm-graph}, and \ref{ass-connected} hold, let $\rho\in\N_{+}$ be given, and suppose that the tuple $(\bx,\calC)$ is a global minimizer of the formulation~\eqref{rel3}, i.e., $(\bx,\calC)\in\mathcal{J}^{\rho}_{g}$. Then, for all $k$, we have that
\begin{equation}
\label{lm:nec1}
\bigcup_{m=1}^{M}\calC_{m}^{k}\neq\emptyset.
\end{equation}
\end{lemma}
\begin{proof} The proof is achieved by contradiction. Suppose the tuple $(\bx,\calC)$ does not satisfy~\eqref{lm:nec1} for all $k$. Then, by Assumption~\ref{ass:data_size}, there exists $\acute{k}$ such that $\calC^{\acute{k}}$ has at least two distinct data points, where
\begin{equation}
\label{lm:nec2}
\calC^{\acute{k}}=\bigcup_{m=1}^{M}\calC_{m}^{k}.
\end{equation}
We now show that the following holds:

\noindent{\bf Claim 1}. There exists $\acute{m}\in\{1,\cdots,M\}$ and a data point $\wy\in\calC_{\acute{m}}^{\acute{k}}$ such that $\wy\neq \bx_{\acute{m}}^{\acute{k}}$.

Now, supposing to the contrary that Claim 1 as noted above does not hold, we must have for all $m=1,\cdots,M$
\begin{equation}
\label{lm:nec3}
\y=\bx_{m}^{\acute{k}}~~\mbox{if $\y\in\calC_{m}^{\acute{k}}$}.
\end{equation}
The only way the above assertion is possible is if each $\calC_{m}^{\acute{k}}$, $m=1,\cdots,M$, contains no more than one distinct data point.
Since, by construction $\calC^{\acute{k}}$ has at least two distinct data points, this implies that there exist $m_{1}$ and $m_{2}$ in $\{1,\cdots,M\}$ such that $\bx_{m_{1}}^{\acute{k}}\neq\bx_{m_{2}}^{\acute{k}}$, which by Assumption~\ref{ass-connected} and properties of the associated Laplacian matrix $L$ further imply that
\begin{equation}
\label{lm:nec4}
\left(L\otimes I_{p}\right)\mathbf{x}^{\acute{k}}\neq\mathbf{0},
\end{equation}
where $\mathbf{x}^{\acute{k}}=\vecc_{m}(\mathbf{x}_{m}^{\acute{k}})$. Now note that, by Proposition~\ref{prop:rel_glob_loc} and the fact that $(\bx,\calC)\in\mathcal{J}^{\rho}_{g}$ (by hypothesis), we have $(\bx,\calC)\in\mathcal{J}^{\rho}$. This, in turn, implies that (by Definition~\ref{def:locmin})
\begin{equation}
\label{lm:nec5}
\bx_{m}^{\acute{k}}=\frac{(1/\rho)\sum_{\y\in\calC_{m}^{\acute{k}}}\y+\sum_{l\in\Omega_{m}}\bx_{l}^{\acute{k}}}{(1/\rho)|\calC_{m}^{\acute{k}}|+|\Omega_{m}|}
\end{equation}
for all $m$. Now, in either case, i.e., as to whether $|\calC_{m}^{\acute{k}}|=0$ or not for a given $m$, combining~\eqref{lm:nec3} and~\eqref{lm:nec5} we obtain
\begin{equation}
\label{lm:nec6}
\bx_{m}^{\acute{k}} = \frac{\sum_{l\in\Omega_{m}}\bx_{l}^{\acute{k}}}{|\Omega_{m}|}
\end{equation}
for each $m$. This, in turn, implies that $\left(L\otimes I_{p}\right)\mathbf{x}^{\acute{k}}=\mathbf{0}$ which clearly contradicts with~\eqref{lm:nec4}. Hence, we conclude that Claim 1 must hold.

Now, by the contradiction hypothesis set up in~\eqref{lm:nec2}, there exists $k_{0}\neq\acute{k}$ such that $\calC_{m}^{k_{0}}=\emptyset$ for all $m$. Further, by Claim 1, there exist $\acute{m}\in\{1,\cdots,M\}$ and a data point $\wy\in\calC_{\acute{m}}^{\acute{k}}$ such that $\wy\neq \bx_{\acute{m}}^{\acute{k}}$. Consider the following potential tuple of cluster centers and clusters $(\wx,\wC)$ defined as follows:
\begin{equation}
\label{lm:nec7}
\wx_{m}^{k} = \left\{\begin{array}{ll}
                                    \bx_{m}^{k} & \mbox{for all $m$ and $k\neq k_{0}$}\\
                                    \wy & \mbox{for all $m$ and $k=k_{0}$},
                                    \end{array}\right.
\end{equation}

\begin{equation}
\label{lm:nec8}
\wC_{m}^{k} = \left\{\begin{array}{ll}
                                    \calC_{\acute{m}}^{\acute{k}}\setminus\{\wy\} & \mbox{if $m=\acute{m}$ and $k=\acute{k}$}\\
                                    \{\wy\} & \mbox{if $m=\acute{m}$ and $k=k_{0}$}\\
                                    \calC_{m}^{k} & \mbox{otherwise}\\
                                    \end{array}\right.
\end{equation}
In other words, the tuple $(\wx,\wC)$ is obtained by essentially transferring the data point $\wy$ from the $\acute{k}$-th cluster at agent $\acute{m}$ to the $k_{0}$-th cluster at $\acute{m}$ (the latter cluster is empty as far as the tuple $(\bx,\calC)$ is concerned), while setting the cluster centers $\wx_{m}^{k_{0}}$, $m=1,\cdots,M$, to be all equal to $\wy$. By directly computing the costs $J^{\rho}(\cdot)$ associated with the tuples $(\bx,\calC)$ and $(\wx,\wC)$ it readily follows that
\begin{equation}
\label{lm:nec9}
J^{\rho}(\wx,\wC)< J^{\rho}(\bx,\calC).
\end{equation}
(Note that by Claim 1 the data point $\wy$ incurs a strictly positive cost in the assignment $(\bx,\calC)$ whereas contributes to zero cost in $(\wx,\wC)$; also, the cluster center agreement costs stay the same by construction~\eqref{lm:nec7}.) Clearly,~\eqref{lm:nec9} contradicts with the fact that the tuple $(\bx,\calC)$ is a global minimizer and the desired assertion follows.
\end{proof}

We now obtain the following result.

\begin{lemma} \label{lm:cvx-hull}
\label{lm:hull} Let Assumptions \ref{ass:data_size}, \ref{ass:comm-graph}, \ref{ass-connected}, and \ref{ass-alpha} hold. Then, for each $\bx\in\mathcal{M}^{\rho}_g$ we have that
\begin{equation}
\label{lm:hull1}
\bx_{m}^{k}\in\cco(\mathcal{D})
\end{equation}
for all $m$ and $k$.
\end{lemma}

%Before proving the lemma we note that we note that by Proposition \ref{prop:rel_glob_loc}, $\calM_g^\rho \subset \calM^\rho$, hence as a special case of the above lemma we see that if $\bx\in \calM_g^\rho$ then $\bx_m^k \in \cco(\calD)$ for all $m$ and $k$.  We now prove Lemma \ref{lm:cvx-hull}.
\begin{proof}
Recall that $\bx \in \calM^\rho_g$ if and only if there exists a $\calC$ such that $(\bx,\calC) \in \calJ^\rho_g$. Moreover, recall that a necessary condition for $(\bx,\calC) \in \calJ^\rho_g$ is that $\bx$ be a global minimizer of the (quadratic) function $\vx \mapsto J^\rho(\vx,\calC)$ (see Remark \ref{remark:gen-min-vs-lloyds}).

For the sake of contradiction, suppose that $\vx\in\calM^\rho_g$ but $\bx_{m}^{k}\notin\cco(\mathcal{D})$ for some $m$ and $k$. Let $\calC$ be a partition such that $(\bx,\calC)\in \calJ^\rho_g$. We will show that there exists an $\hat \vx$ such that $J^\rho(\hat \vx,\calC) < J^\rho(\bx,\calC)$ which contradicts the hypothesis that $(\bx,\calC)\in \calJ^\rho_g$ (or equivalently, that $\bx\in \calM^\rho_g$).

%We now proceed with the proof.
Let $\calI =\{(m,k):~ \bx_m^k \notin\cco(\calD)\}$ denote the set of non-compliant indices and let $P_{\calD}:\R^p\to\cco(\calD)$ be the conventional projection operator onto $\cco(\calD)$ given by
$$P_\calD(\vx)= \{\vz\in \cco(\calD): \|\vx - \vz\| \leq \|\vx - \vz'\| ~\forall \vz'\in \cco(\calD)\},$$
Note that since $\cco(\calD)$ is convex and compact, $P_\calD(\vx)$ is non-empty and single valued for all $\vx\in \R^p$ and, moreover, $P_\calD$ is nonexpansive in the sense that for any $\vx,\vy\in\R^p$,
\begin{equation} \label{eq:non-expansive}
\|P_\calD(\vx) - P_\calD(\vy)\| \leq \|\vx - \vy\|.
\end{equation}

%We will now show that for any $\epsilon > 0$ we can construct a vector $\hat \vx\in B_\epsilon(\bx)\subset \R^{MNp}$ for which $\myQ^\rho(\hat \vx) \leq \myQ^\rho(\bx)$. (Note somewhere that $x\in \calM^\rho_g$ is equivalent to $(x,\calC)\in\calJ^\rho_g$ for some partition $\calC$ which holds if and only if $x\mapsto \calJ^\rho_g(x,\calC)$ is a local (global?) min for fixed $\calC$. Thus, showing there exists an $\hat x$ s.t. $J^\rho(\hat x,\calC) < J^\rho(\bx,\calC)$ implies $(\bx,\calC)\not\in \calJ^\rho_g$, which implies that $\bx\not\in \calM^\rho_g$, which contradicts our hypothesis.)

Let $\hat \vx = \vecc_{m,k}(\hat \vx_m^k)$, where for each $(m,k)$
$$
\hat \vx_m^k = P_\calD(\bx_m^k).
$$
Let $\vz$ be an arbitrary element of $\cco(\calD)$.
Since $\cco(\calD)$ is convex, for each $m$, $k$ we have that
$$
(\vz - \hat \vx_m^k)^T (\bx_m^k - \hat \vx_m^k) \leq  0.
$$
%where for two vectors $a,b\in\R^p$, $\langle a,b \rangle= a^T b$ denotes the standard inner product.
By the law of cosines this gives
\begin{align}
\|\vz  - \bx_m^k\|^2 & = \| \vz - \hat \vx_m^k\|^2 + \|\bx_m^k - \hat \vx_m^k\|^2 - 2(\vz - \hat \vx_m^k)^T(\bx_m^k - \hat \vx_m^k)\\
& \geq \| \vz - \hat \vx_m^k\|^2 + \|\bx_m^k - \hat \vx_m^k\|^2
\end{align}
for any $m=1,\ldots,M$, $k=1,\ldots, K$.
Since $\cco(\calD)$ is compact, for each $(m,k) \in \calI$ we have $\|\bx_m^k - \hat \vx_m^k\| > 0$, and hence
\begin{equation} \label{eq:law-of-cosines}
\|\vz  - \bx_m^k\|^2 > \| \vz - \hat \vx_m^k\|^2.
\end{equation}
%%Since $\hat \vx_m^k = P_\calD(\bx_m^k)$,
%For each $(m,k) \in \calI$ we have that $\|\vz_m^k - \bx_m^k\|^2 > \|\vz_m^k - \hat \vx_m^k\|^2$. (JUSTIFY.) Note that the inequality is strict due to the assumption that $\bx_m^k\not\in\cco(\calD)$. Thus, by the theorem of Pythagoras we see that
%\begin{align}
%\|y - \bx_m^k\|^2 & = \|y - \vz_m^k\|^2 + \|\vz_m^k - \bx_m^k\|^2\\
%& > \|y - \vz_m^k\|^2 + \|\vz_m^k - \hat \vx_m^k\|^2\\
%& = \|y - \hat \vx_m^k\|^2
%\end{align}
%for all $y\in \calD$ and $(m,k) \in\calI$.

Furthermore, by \eqref{eq:non-expansive} we have that $\|\hat \vx_m^k - \hat \vx_\ell^k\|^2 \leq \|\bx_m^k - \bx_\ell^k\|^2$ for all $m,\ell\in\{1,\ldots,M\}$ and $k=1,\ldots,K$.
Thus we see that
\begin{align}
J^\rho(\bx, \calC) & = \sum_{m=1}^M \sum_{k=1}^K \sum_{y\in \calC_m^k} \|y-\bx_m^k\|^2 + \sum_{m=1}^M \sum_{k=1}^K \sum_{\ell\in \Omega_m} \|\bx_m^k - \bx_\ell^k\|^2\\
& < \sum_{m=1}^M \sum_{k=1}^K \sum_{y\in \calC_m^k} \|y-\hat \vx_m^k\|^2 + \sum_{m=1}^M \sum_{k=1}^K \sum_{\ell\in \Omega_m} \|\hat \vx_m^k - \hat \vx_\ell^k\|^2\\
& = J^\rho(\hat \vx,\calC),
\end{align}
where the strict inequality follows from \eqref{eq:law-of-cosines} and the fact that $\calI \not = \emptyset$. This completes the proof.
\end{proof}

As an immediate consequence of Proposition~\ref{prop:rel_glob_loc} and Lemma~\ref{lm:hull} we obtain the following result.
\begin{corollary}
\label{corr:cohull}
Let Assumptions \ref{ass:data_size}, \ref{ass:comm-graph}, and \ref{ass-connected} hold.
%and denote by $\oM^{\rho}$ the subset of $\mathcal{M}^{\rho}$ such that
Then, we have that
\begin{equation}
\label{corr:chull2}
\mathcal{M}_{g}^{\rho}\subset\oM^{\rho}.
\end{equation}
\end{corollary}

\subsection{Limiting behavior of the sets $\mathcal{M}_{g}^{\rho}$ and $\oM^{\rho}$ as $\rho\rightarrow\infty$}
\label{subsec:limsets}

We start with the following result that quantifies the deviation from consensus of the agent cluster center estimates at a generalized minimum. We note that Corollary \ref{corr:cohull} allows us to study properties of $\calM^\rho_g$ by studying $\oM^\rho$. This is the approach we will take here.
\begin{lemma}
\label{lm:rhocons} Let Assumptions \ref{ass:data_size}, \ref{ass:comm-graph}, and \ref{ass-connected} hold and let $\bx\in\oM^{\rho}$. Then, for all $k$, we have that
\begin{equation}
\label{lm:rhocons1}
\left\|\bx_{m}^{k}-\bx_{l}^{k}\right\|\leq \frac{4\sqrt{M}R_{0}|\mathcal{D}|}{\rho\lambda_{2}(L)}
\end{equation}
for each pair $(m,l)$ of agents, where
\begin{equation}
\label{lm:rhocons2}
R_{0}=\max_{\mathbf{v}\in\cco(\mathcal{D})}\|\mathbf{v}\|
\end{equation}
and $\lambda_{2}(L)$ denotes the second largest eigenvalue of the communication network graph Laplacian $L$.
\end{lemma}
\begin{proof}
Let $\calC$ be such that $(\bx,\calC)\in\mathcal{J}^{\rho}$. By Definition~\ref{def:locmin} we have that
\begin{equation}
\label{lm:rhocons3}
\bx_{m}^{k}=\frac{(1/\rho)\sum_{\y\in\calC_{m}^{k}}\y + \sum_{l\in\Omega_{m}}\bx_{l}^{k}}{(1/\rho)|\calC_{m}^{k}|+|\Omega_{m}|}
\end{equation}
for all $m$ and $k$. Rearranging~\eqref{lm:rhocons3} we obtain
\begin{align}
\label{lm:rhocons4}
\left\||\Omega_{m}|\bx_{m}^{k}-\sum_{l\in\Omega_{m}}\bx_{l}^{k}\right\| & \leq (1/\rho)\left\|\sum_{\y\in\calC}\y\right\|+(1/\rho)|\calC_{m}^{k}|\left\|\bx_{m}^{k}\right\|\\
\leq & \frac{2R_{0}|\mathcal{D}|}{\rho},
\end{align}
where $R_{0}$ is defined in~\eqref{lm:rhocons2} and we used the fact that $|\calC_{m}^{k}|\leq |\mathcal{D}|$. Fixing $k$ and stacking~\eqref{lm:rhocons4} over $m$, we obtain
\begin{equation}
\label{lm:rhocons5}
\left\|\left(L\otimes I_{p}\right)\bx^{k}\right\|\leq \frac{2\sqrt{M}R_{0}|\mathcal{D}|}{\rho},
\end{equation}
where $\bx^{k}=\vecc_{m}(\bx_{m}^{k})$. Denoting by $\underline{\bx}^{k}\in\mathbb{R}^{p}$ the average
\begin{equation}
\label{lm:rhocons6}
\underline{\bx}^{k}=(1/M)\sum_{m=1}^{M}\bx_{m}^{k}
\end{equation}
and noting that $(L\otimes I_{p})\underline{\bx}^{k}=\mathbf{0}_{p}$ we have
\begin{equation}
\label{lm:rhocons7}
\left\|\left(L\otimes I_{p}\right)\left(\bx^{k}-\mathbf{1}_{p}\otimes\underline{\bx}^{k}\right)\right\|\leq \frac{2\sqrt{M}R_{0}|\mathcal{D}|}{\rho}.
\end{equation}
Now note that the vector $\left(\bx^{k}-\mathbf{1}_{p}\otimes\underline{\bx}^{k}\right)$ is orthogonal to the consensus subspace, i.e., the subspace $\{\vx\in \R^{KMp}: \vx = \textbf{1}_{M}\otimes \va \mbox{ for some } \va\in \R^{Kp}\}$, and hence we have that
\begin{equation}
\label{lm:rhocons8}
\left\|\left(L\otimes I_{p}\right)\left(\bx^{k}-\mathbf{1}_{p}\otimes\underline{\bx}^{k}\right)\right\|\geq\lambda_{2}(L)\left\|\bx^{k}-\mathbf{1}_{p}\otimes\underline{\bx}^{k}\right\|,
\end{equation}
where $\lambda_{2}(L)>0$ by Assumption~\ref{ass-connected}. By~\eqref{lm:rhocons7}--\eqref{lm:rhocons8} we obtain
\begin{align}
\label{lm:rhocons9}
\left\|\bx_{m}^{k}-\bx_{l}^{k}\right\| & \leq \left\|\bx_{m}^{k}-\underline{\bx}^{k}\right\|+\left\|\bx_{l}^{k}-\underline{\bx}^{k}\right\|\\
& \leq 2\left\|\bx^{k}-\mathbf{1}_{p}\otimes\underline{\bx}^{k}\right\|\\
&\leq \frac{2}{\lambda_{2}(L)}\left\|\left(L\otimes I_{p}\right)\left(\bx^{k}-\mathbf{1}_{p}\otimes\underline{\bx}^{k}\right)\right\|\\
& \leq\frac{4\sqrt{M}R_{0}|\mathcal{D}|}{\rho\lambda_{2}(L)}.
\end{align}
\end{proof}

We now quantify as a function of $\rho$ the \emph{optimality gap} between the $K$-means formulation~\eqref{Kmeans} and its relaxation~\eqref{rel6}.
\begin{lemma}
\label{lm:gloptgap}
Let Assumptions \ref{ass:data_size}, \ref{ass:comm-graph}, and \ref{ass-connected} hold and let $\bx$ be a global minimizer of $\myQ^{\rho}(\cdot)$, i.e., $\bx\in\mathcal{M}^{\rho}_{g}$. For each $m$, denote by $\bx_{m}$ the $K$-tuple $\{\bx_{m}^{1},\cdots,\bx_{m}^{K}\}$. Then, we have that
\begin{equation}
\label{lm:gloptgap1}
\mathcal{F}(\bx_{m})\leq\mathcal{F}^{\ast} + \frac{16\sqrt{M}R^{2}_{0}|\mathcal{D}|^{2}}{\rho\lambda_{2}(L)},
\end{equation}
where $\mathcal{F}^{\ast}$ is the global minimum value of~\eqref{Kmeans} and $R_{0}$ is defined in~\eqref{lm:rhocons2}.
\end{lemma}

\begin{proof}
Let $\bz$ be a global minimizer of~\eqref{Kmeans}, i.e., $\mathcal{F}(\bz)=\mathcal{F}^{\ast}$. Note that by definition of the cost functions
\begin{equation}
\label{lm:gloptgap2}
\mathcal{F}(\bz)=\rho\myQ^{\rho}(\bz,\cdots,\bz).
\end{equation}
(The agreement part of the cost in~\eqref{rel600} vanishes when all agents employ common cluster center estimates.) Since $\bx\in\mathcal{M}^{\rho}_{g}$ it then follows that
\begin{equation}
\label{lm:gloptgap3}
\myQ^{\rho}(\bx)\leq (1/\rho)\mathcal{F}(\bz).
\end{equation}
%Let $\underline{\bx}=\vecc_{k}(\underline{\bx}^{k})$ be the agent-wise average of $\bx$, i.e., for all $k$, $\underline{\bx}^{k} = (1/M).(\bx_{1}^{k}+\cdots+\bx_{M}^{k})$. We have, by arguments similar to the derivation in~\eqref{lm:gloptgap2}
%\begin{equation}
%\label{lm:gloptgap4}
%\mathcal{F}(\underline{\bx})=\rho\myQ^{\rho}(\underline{\bx},\cdots,\underline{\bx}).
%\end{equation}
For each $m$ and data point $\y\in\mathcal{D}_{m}$, let $k_{m,n}\in\{1,\cdots,K\}$ be such that
\begin{equation}
\label{lm:gloptgap5}
\|\y-\bx_{m}^{k_{m,n}}\|\leq\|\y-\bx_{m}^{k}\|~~~\forall k\in \{1,\cdots,K\}.
\end{equation}
Now fix $m_{0}\in\{1,\cdots,M\}$ and consider the $K$-tuple $\bx_{m_{0}}=\{\bx_{m_{0}}^{1},\cdots,\bx_{m_{0}}^{K}\}$. Noting that for all $m$ and $k$ (see Lemma~\ref{lm:rhocons})
\begin{equation}
\label{lm:gloptgap6}
\|\bx_{m}^{k}-\bx_{m_{0}}^{k}\|\leq \frac{4\sqrt{M}R_{0}|\mathcal{D}|}{\rho\lambda_{2}(L)},
\end{equation}
we have that
\begin{align}
\label{lm:gloptgap7}
&\|\y-\bx_{m_{0}}^{k_{m,n}}\|^{2}-\|\y-\bx_{m}^{k_{m,n}}\|^{2}\\
&=\left(\|\y-\bx_{m_{0}}^{k_{m,n}}\|+\|\y-\bx_{m}^{k_{m,n}}\|\right)\left(\|\y-\bx_{m_{0}}^{k_{m,n}}\|-\|\y-\bx_{m}^{k_{m,n}}\|\right)\\
&\leq\left(2\|\y+\|\bx_{m}^{k_{m,n}}\|+\|\bx_{m_{0}}^{k_{m,n}}\|\right)\left\|\bx_{m}^{k_{m,n}}-\bx_{m_{0}}^{k_{m,n}}\right\|\\
&\leq (4R_{0})\frac{4\sqrt{M}R_{0}|\mathcal{D}|}{\rho\lambda_{2}(L)}=\frac{16\sqrt{M}R^{2}_{0}|\mathcal{D}|}{\rho\lambda_{2}(L)},
\end{align}
where we also use the fact that $\bx_{m}^{k_{m,n}},\bx_{m_{0}}^{k_{m,n}}\in\cco(\mathcal{D})$ (see Lemma~\ref{lm:hull}). Summing over all $\y\in\mathcal{D}$ both sides of~\eqref{lm:gloptgap7}, we obtain the estimate
\begin{align}
\label{lm:gloptgap8}
\sum_{\y\in\mathcal{D}}\|\y-\bx_{m_{0}}^{k_{m,n}}\|^{2}\leq \sum_{\y\in\mathcal{D}}\|\y-\bx_{m}^{k_{m,n}}\|^{2}+\left(|\mathcal{D}|\right)\frac{16\sqrt{M}R^{2}_{0}|\mathcal{D}|}{\rho\lambda_{2}(L)}\\
=\sum_{m=1}^{M}\sum_{\y\in\mathcal{D}_{m}}\min_{k=1,\cdots,K}\|\y-\bx_{m}^{k}\|^{2}+\frac{16\sqrt{M}R^{2}_{0}|\mathcal{D}|^{2}}{\rho\lambda_{2}(L)}.
\end{align}
Noting that
\begin{equation}
\label{lm:gloptgap9}
\mathcal{F}(\bx_{m_{0}})\leq \sum_{\y\in\mathcal{D}}\|\y-\bx_{m_{0}}^{k_{m,n}}\|^{2}
\end{equation}
and
\begin{equation}
\label{lm:gloptgap10}
\sum_{m=1}^{M}\sum_{\y\in\mathcal{D}_{m}}\min_{k=1,\cdots,K}\|\y-\bx_{m}^{k}\|^{2}\leq\rho\myQ^{\rho}(\bx),
\end{equation}
we have by~\eqref{lm:gloptgap8}
\begin{equation}
\label{lm:gloptgap11}
\mathcal{F}(\bx_{m_{0}})\leq \rho\myQ^{\rho}(\bx)+\frac{16\sqrt{M}R^{2}_{0}|\mathcal{D}|^{2}}{\rho\lambda_{2}(L)}.
\end{equation}
By~\eqref{lm:gloptgap3} it follows that
\begin{equation}
\label{lm:gloptgap12}
\mathcal{F}(\bx_{m_{0}})\leq\mathcal{F}^{\ast}+\frac{16\sqrt{M}R^{2}_{0}|\mathcal{D}|^{2}}{\rho\lambda_{2}(L)}.
\end{equation}
Since~\eqref{lm:gloptgap12} holds for all $m_{0}\in\{1,\cdots,M\}$, the desired assertion follows.
\end{proof}

We now prove Theorem \ref{lm:gloptconv}.
\begin{proof}[\textbf{Proof of Theorem \ref{lm:gloptconv}}]
Note that the set $\mathcal{Z}_{g}$ is compact and the sequence $\{\bx^{\rho}_{m}\}_{\rho\in\N_{+}}$, for each $m$, is bounded (see Lemma~\ref{lm:hull}). It then suffices to show that each limit point $\overline{\mathbf{x}}_{m}$ of $\{\bx^{\rho}_{m}\}_{\rho\in\N_{+}}$ belongs to $\mathcal{Z}_{g}$.

To this end, let $\{\bx^{\rho_{s}}_{m}\}_{s\geq 0}$ be a convergent subsequence of $\{\bx^{\rho}_{m}\}$ such that $\bx^{\rho_{s}}_{m}\rightarrow\overline{\mathbf{x}}_{m}$ as $s\rightarrow\infty$. Now since $\bx^{\rho_{s}}\in\mathcal{M}_{g}^{\rho_{s}}$, by Lemma~\ref{lm:gloptgap} we have
\begin{equation}
\label{lm:gloptconv2}
\mathcal{F}(\bx^{\rho_{s}}_{m})\leq\mathcal{F}^{\ast}+\frac{16\sqrt{M}R^{2}_{0}|\mathcal{D}|^{2}}{\rho_{s}\lambda_{2}(L)}
\end{equation}
for all $s$, where $\mathcal{F}^{\ast}$ is the global minimum value of~\eqref{Kmeans}. Now note that $\mathcal{F}(\cdot)$ is a continuous function and hence $\mathcal{F}(\bx^{\rho_{s}}_{m})\rightarrow\mathcal{F}(\overline{\mathbf{x}}_{m})$ as $s\rightarrow\infty$. Since $\rho_{s}\rightarrow\infty$ as $s\rightarrow\infty$ (by definition of a subsequence), by taking the limit as $s\rightarrow\infty$ on~\eqref{lm:gloptconv2} we obtain
\begin{equation}
\label{lm:gloptconv3}
\mathcal{F}(\overline{\mathbf{x}}_{m})\leq \mathcal{F}^{\ast}.
\end{equation}
Since the $K$-tuple $\overline{\mathbf{x}}_{m}=\{\overline{\mathbf{x}}_{m}^{1},\cdots,\overline{\mathbf{x}}_{m}^{K}\}$ is a feasible solution for the minimization formulation~\eqref{Kmeans}, we conclude that $\mathcal{F}(\overline{\mathbf{x}}_{m})=\mathcal{F}^{\ast}$ and $\overline{\mathbf{x}}_{m}$ is a global minimizer of~\eqref{Kmeans}, i.e., $\overline{\mathbf{x}}_{m}\in\mathcal{Z}_{g}$. We have thus shown that each limit point of the sequence $\{\bx^{\rho}_{m}\}_{\rho\in\N_{+}}$ belongs to $\mathcal{Z}_{g}$ which establishes the assertion.
\end{proof}

Finally, we note that using similar arguments Lemma~\ref{lm:gloptconv} can be strengthened to achieve the following uniform convergence property.
\begin{corollary}
\label{corr:gloptconv}
Let Assumptions \ref{ass:data_size}, \ref{ass:comm-graph}, and \ref{ass-connected} hold. Then, we have that
\begin{equation}
\label{corr:gloptconv1}
\lim_{\rho\rightarrow\infty}\sup_{\bx^{\rho}\in\mathcal{M}_{g}^{\rho}}\max_{m=1,\cdots,M}d\left(\bx_{m}^{\rho},\mathcal{Z}_{g}\right)=0,
\end{equation}
where for each $\rho$, $\bx^{\rho}\in\mathcal{M}_{g}^{\rho}$ and $m$, the quantity $\bx_{m}^{\rho}$ denotes the $K$-tuple $\{\bx_{m}^{\rho,1},\cdots,\bx_{m}^{\rho,K}\}$ of cluster center estimates at an agent $m$.

In particular, we have that, for $\Vap>0$, there exists $\rho_{\Vap}\doteq\rho_{\Vap}(\mathcal{D},\mathcal{G})$, a function of the data set $\mathcal{D}$ and the inter-agent communication topology $G$ only, such that
\begin{equation}
\label{corr:gloptconv2}
d\left(\bx_{m}^{\rho},\mathcal{Z}_{g}\right)\leq\Vap
\end{equation}
for all $m$, $\rho\geq\rho_{\Vap}$ and $\bx^{\rho}\in\mathcal{M}_{g}^{\rho}$.
\end{corollary}

\section{Conclusions}
The paper considered the problem of $K$-means clustering in networked IoT-type settings where data is distributed among nodes of the network. The networked $K$-means ($NK$-means) algorithm was proposed as a decentralized method for computing Lloyd’s minima in such settings. Formal convergence guarantees for the $NK$-means algorithm were presented in Theorems \ref{lm:conv_to_Z}--\ref{th:conv} and Corollary \ref{cor:finite-rho}.
%The set of limit points of the $NK$-means algorithm is the set of \emph{generalized} Lloyd’s minima which can be made as close as desired to the set of classical Lloyd’s minima through appropriate choice of the design parameter $\rho$.

The present work focused on $K$-means clustering, i.e., clustering with an $\ell_2$ distortion. Future work may consider extensions of the techniques considered here to clustering with $\ell_1$ distortion (i.e., $K$-medians clustering) which may be more stable to outliers, or, more generally, clustering with $\ell_p$ distortions or arbitrary convex distortions.

The present work did not address the issue of initialization or its effects on the obtained clustering. In general, the efficiency of the clustering obtained by Lloyd’s algorithm (and related procedures) can be sensitive to the initial choice of cluster heads \citep{milligan1980examination}.
%This issue can be mitigated by judicious choice of the initial cluster heads.
Extensions of the current work may consider techniques for initialization to ensure efficient clusterings are obtained (e.g., variants of \citep{arthur2007k,ostrovsky2006effectiveness,bahmani2012scalable} adapted to the network-based framework). Other extensions include specializing to settings where the data and/or statistics are structured in order to achieve stronger guarantees (e.g.,~\citep{pollard1981strong,serinko1992weak}), as well as extensions to obtain other types of minima such as solutions in the sense of Hartigan \citep{telgarsky2010hartigan}.

\section*{Appendix A}
\noindent \textbf{Glossary of Frequently used Notation.}
For convenience, this appendix collects symbols used frequently throughout the paper.
\begin{itemize}
\item $M$ is the number of agents
\item $\calD_m$ is the data set of agent $m$
\item $\calD$ is the joint dataset $\calD = \calD_1\cup\cdots\cup \calD_M$
\item $N_m = |\calD_m|$
\item $N = |\calD|$
\item $\C_m$ is the set of size-$K$ partitions of $\calD_m$
\item $\Omega_m$ is the (communication) neighborhood of agent $m$ relative to the graph $G$
\item $\calF(\vz)$, $\calH(\vz,\calP)$ are cost functions associated with the (centralized) $K$-means formulation
\item $\myQ^\rho(\vx)$, $J^\rho(\vx,\calC)$ are cost functions associated with the generalized (distributed) $NK$-means formulation
\item $\calL$ is the set of partitions and cluster heads corresponding to Lloyd's minima
\item $\calZ$ is the set of cluster heads corresponding to Lloyd's minima
\item $\calL_g$ is the set of global minima of $\calH(\vz,\calP)$
\item $\calZ_g$ is the set of global minima of $\calF(\vz)$

\item $\calJ^\rho$ is set of generalized Lloyd's minima of $J^\rho(\vx,\calC)$
\item $\calM^\rho = \{\vx = \{\vx^1,\ldots,\vx^K\}:~ (\vx,\calC)\in\calJ^\rho \mbox{ for some } \calC\}$
%of analogous to $\calJ^\rho$, sans partitions (generalized minima of $\myQ^\rho$)
\item $\calJ^\rho_g$ is the set of global minima of $J^\rho(\vx,\calC)$
\item $\calM^\rho_g$ is the set of global minima of $\myQ^\rho(\vx)$
\item $\oM^{\rho}=\left\{\bx\in\mathcal{M}^{\rho}~:~\bx_{m}^{k}\in\cco(\mathcal{D})~~\forall m,k\right\}$
\item $\OZ=\left\{\vx\in\mathcal{Z}~:~\vx^{k}\in\cco(\mathcal{D})~~\forall k\right\}$
\end{itemize}

\vskip 0.2in
%\bibliography{sample}
\bibliographystyle{abbrvnat}
\bibliography{CentralBib,Refs-NSF_Data_Science,opt_refs,glrt,dsprt}

\end{document}